\documentclass{article}

\PassOptionsToPackage{numbers,compress}{natbib}
\usepackage{lmodern}
\usepackage[preprint]{neurips_2026}

\usepackage[utf8]{inputenc}
\usepackage[T1]{fontenc}
\usepackage[hidelinks]{hyperref}
\usepackage{url}
\usepackage{booktabs}
\usepackage{amsfonts}
\usepackage{nicefrac}
\usepackage[english]{babel}
\usepackage{microtype}
\usepackage[dvipsnames]{xcolor}
\usepackage[singlelinecheck=false]{caption}
\usepackage{multirow}
\usepackage{siunitx}
\usepackage{colortbl}
\usepackage{rotating}
\usepackage{adjustbox}
\usepackage{tabularray}
\usepackage{hhline}
\usepackage{makecell}
\usepackage{mdframed}

\usepackage{subcaption}

\usepackage{array}

\usepackage{amssymb}
\usepackage{graphicx}
\usepackage{enumitem}
\usepackage{mathtools}

\usepackage{amsmath} 
\usepackage{amsthm}

\newtheorem{definition}{Definition}[section]

\newtheorem{assumption}{Assumption}[section]
\newtheorem{theorem}{Theorem}[section]
\newtheorem{lemma}{Lemma}[section]

\usepackage{xspace}

\title{Expert Upcycling: Shifting the Compute-Efficient Frontier of Mixture-of-Experts}

\author
{\textbf{Chaitanya Dwivedi}\textsuperscript{1,2} \quad \textbf{Binxuan Huang}\textsuperscript{1,3} \quad \textbf{Himanshu Gupta}\textsuperscript{1} \\ \textbf{Pratik Jayarao}\textsuperscript{1,2}  \quad \textbf{Neeraj Varshney}\textsuperscript{1}  \quad \textbf{Bing Yin}\textsuperscript{1}  \\
\small{\textsuperscript{1}Amazon Stores Foundation AI \quad \textsuperscript{2}Carnegie Mellon University \quad \textsuperscript{3}Anthropic} \\
\tt\small {cdwivedi}@alumni.cmu.edu \\
}

\begin{document}

\maketitle

\begin{abstract}
Mixture-of-Experts (MoE) has become the dominant
architecture for scaling large language models:
frontier models routinely decouple total parameters
from per-token computation through sparse expert
routing.
Scaling laws show that under fixed active computation,
model quality scales predictably with total parameters,
and MoEs realize this by increasing expert count.
However, training large MoEs is expensive, as memory
requirements and inter-device communication both scale
with total parameter count.
We propose \emph{expert upcycling}, a method for
progressively expanding MoE capacity by increasing the
number of experts during continued pre-training (CPT).
Given a trained $E$-expert model, the upcycling
operator constructs an $mE$-expert model through expert
duplication and router extension while holding top-$K$
routing fixed, preserving per-token inference cost.
Duplication provides a \emph{warm initialization}: the
expanded model inherits the source checkpoint's learned
representations, starting from a substantially lower
loss than random initialization. Subsequent CPT then
breaks the symmetry among duplicated experts to drive
specialization.
We formalize the upcycling operator and develop a
theoretical framework decomposing the quality gap into
a capacity term and an initialization term. We further
introduce \emph{utility-based expert selection}, which
uses gradient-based importance scores to guide
non-uniform duplication, more than tripling gap closure
when CPT is limited. In our 7B$\to$13B total
parameter experiments, the upcycled model achieves lower
validation loss than the fixed-size baseline while saving
${\sim}$32\% of GPU hours at 50\% CPT. Extending to 100\% CPT
saves ${\sim}$24\% and closes the downstream accuracy gap
to within 0.3 points (average across 11 benchmarks). Comprehensive
ablations across model scales, activation ratios, MoE
architectures, and training budgets yield a practical
recipe for deploying expert upcycling, establishing it
as a principled, compute-efficient alternative to
training large MoE models from scratch\footnote{\label{fn:code}We release our code and training configurations at \url{https://github.com/amazon-science/expert-upcycling}.}.
\end{abstract}

\begin{figure*}[t]
  \centering
  \includegraphics[width=0.90\linewidth]{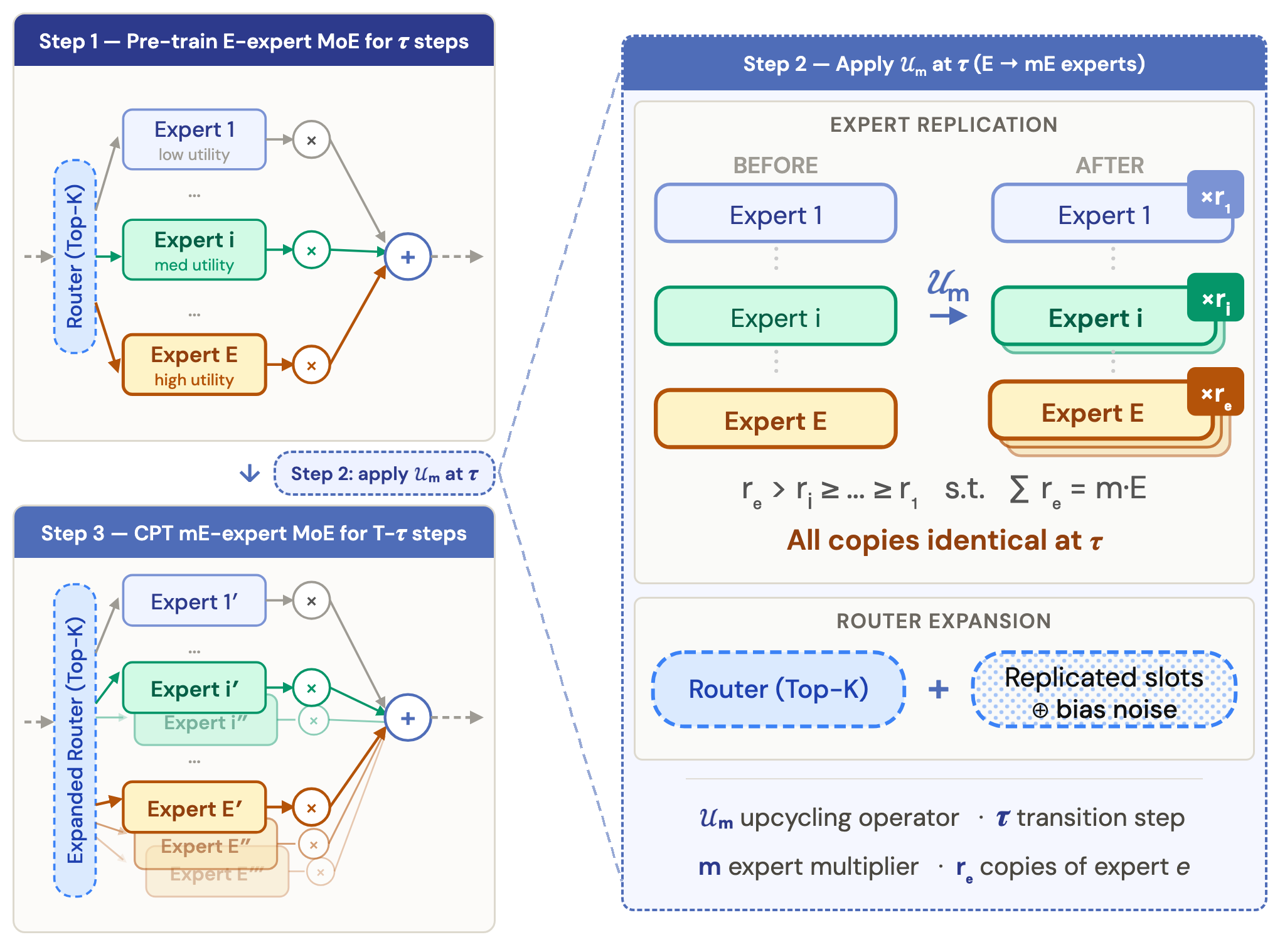}
  \caption{
    \textbf{Overview of the expert upcycling procedure.}
    \textbf{Step 1:} Pre-train an $E$-expert MoE for $\tau$ steps.
    \textbf{Step 2:} Apply the upcycling operator $\mathcal{U}_m$ at step $\tau$: each expert $e$ is replicated $r_e \geq 1$ times (high-utility experts receive more copies, $r_E > r_i \geq \cdots \geq r_1$, s.t.\ $\sum r_e = m \cdot E$), and the router is extended with replicated slots plus bias noise. All copies are identical at $\tau$, providing a warm initialization.
    \textbf{Step 3:} Continue pre-training on the expanded $mE$-expert model for $T{-}\tau$ steps; stochastic gradient diversity breaks symmetry among duplicates, driving specialization. Top-$K$ routing is fixed throughout, so active parameters and per-token compute are unchanged.
  }
  \label{fig:upcycling}
\end{figure*}

\section{Introduction}
\label{sec:intro}

Mixture-of-Experts (MoE) models have become the
dominant architecture for scaling language models
efficiently~\citep{shazeer2017outrageously,
jiang2024mixtral, deepseekai2024deepseekv3, kimik2}.
By routing each token to $K$ out of $E$ total experts,
MoEs decouple total parameters from per-token compute.
Scaling-law analyses show that under fixed active
computation, model quality scales predictably with
total parameters~\citep{ludziejewski2025jointmoe,
abnar2025optimalsparsity}, with the activation ratio
($K/E$) identified as the primary driver of MoE
efficiency over dense
models~\citep{tian2025greaterleverage}.
MoEs realize this by increasing expert count at fixed
$K$, expanding capacity without increasing inference
cost.
Frontier MoE models have pushed this aggressively:
Qwen3 activates 22B of 235B total
parameters~\citep{qwen3}, DeepSeek-V3 activates 37B of
671B~\citep{deepseekai2024deepseekv3}, and Kimi~K2
activates 32B of
1T~\citep{kimik2}, all matching or exceeding dense
models many times their active size.

Despite these favorable scaling properties, training
MoEs with a large number of total parameters from
scratch is expensive.
All expert weights, gradients, and optimizer states
must reside in accelerator memory regardless of how
few experts are active per token, so memory
requirements, and therefore the number of GPUs needed,
scale with the total parameter
count~\citep{zadouri2023moe, ludziejewski2025jointmoe}.
Further, distributing experts across devices introduces
all-to-all communication that can consume 45--50\% of
total training time on standard GPU
clusters~\citep{nie2024lshmoe}.
Both costs grow with $E$, making it progressively more
expensive to train MoEs at the low activation ratios
that scaling laws recommend.

This tension motivates \emph{expert upcycling}, a
capacity-expansion strategy for MoEs that obtains the quality
benefits of a larger MoE without paying its full
training cost from scratch.
Rather than committing to the full expert count from
step~0, training begins with a smaller $E$-expert
model.
At a chosen transition step~$\tau$, the upcycling
operator expands the model to $mE$ experts by
duplicating existing experts and extending the router,
increasing total parameters while holding active
parameters and per-token FLOPs fixed.
This two-phase strategy is strictly cheaper than
fixed-size training: both phases process the same
tokens, but the first $\tau$ steps execute on the
smaller model.
Unlike a randomly initialized larger MoE, the upcycled
model inherits the source checkpoint's learned
representations, providing a \emph{warm initialization}
that starts at approximately the same loss. Continued
pre-training (CPT) then breaks the symmetry among
duplicated experts, driving them to specialize.

To our knowledge, expert upcycling is the first method
to progressively grow MoE capacity during training
while preserving inference cost by holding top-$K$
routing fixed.
This distinguishes it from dense progressive training
methods~\citep{Chen2016Net2Net, du2024stackingtransformers}
and width-expansion approaches for
MoEs~\citep{yu2026sparkling}, which increase active
parameters and thus inference cost, and from sparse
upcycling~\citep{komatsuzaki2022sparse}, which
converts a dense model into an MoE but does not
address capacity expansion within already-sparse
architectures.

Our main contributions are:
\begin{enumerate}[nosep, label=(\roman*)]
  \item We propose \textbf{expert upcycling}, a
    method for progressively expanding MoE capacity
    by increasing the number of experts during
    continued pre-training, and formalize the
    duplication and router-extension operator
    (\S\ref{sec:method}).
    In our 7B$\to$13B total parameter experiments,
    the upcycled model achieves lower validation loss
    than the fixed-size baseline while saving ${\sim}$32\%
    of GPU hours at 50\% CPT; extending to 100\% CPT
    saves ${\sim}$24\% and closes the downstream
    accuracy gap to within 0.3 points (average across
    11 benchmarks).
    Across activation ratios, expert upcycling
    (MoE$\to$MoE) consistently outperforms sparse
    upcycling~\citep{komatsuzaki2022sparse}
    (dense$\to$MoE) as target activation ratio decreases.
  \item Within expert upcycling, we introduce
    \textbf{utility-based expert selection}, a novel
    operator that uses gradient-based importance
    scores to guide non-uniform duplication.
    Utility-based selection consistently outperforms
    uniform duplication, more than tripling gap closure
    when CPT budget is limited
    (\S\ref{sec:utility_upcycling}).
  \item We develop a \textbf{theoretical framework}
    that decomposes the quality gap into a capacity
    term and an initialization term, yielding testable
    predictions for when upcycling succeeds, and
    validate these predictions through
    \textbf{comprehensive experiments} across model
    scales (154M--7B total parameters), activation
    ratios (3--50\%), MoE architectures, training
    budgets, and transition points,
    deriving a practical recipe for practitioners
    (\S\ref{sec:method}, \S\ref{sec:results}). Code is provided as supplemental material and will be released on GitHub upon publication (footnote~\ref{fn:code}).
\end{enumerate}

\noindent
Together, these results establish expert upcycling as
a principled, compute-efficient paradigm for training
sparse models where progressive capacity
expansion is not merely a fallback for reusing
existing checkpoints, but has the potential to become the recommended training
strategy from the outset.

\section{Related work} 
\label{sec:related_work}

\paragraph{Scaling Mixture-of-Experts models.}
Sparsely-gated MoE layers enable high-capacity models with limited per-token computation~\citep{shazeer2017outrageously}, and recent work has scaled this paradigm to open-source frontier models~\citep{du2022glam, fedus2022switch, Lepikhin2020GShard, jiang2024mixtral, metalllama4, deepseekai2024deepseekv3, qwen3, kimik2, glm45}.
Across these systems the trend is consistent: total parameters grow aggressively while active parameters per token remain fixed, directly instantiating the scaling-law prediction that lower activation ratios yield better quality-per-FLOP trade-offs~\citep{tian2025greaterleverage, ludziejewski2025jointmoe, krajewski2024finegrained}.
Expert upcycling exploits precisely this property: by expanding total expert count mid-training while holding top-$K$ fixed, it captures the quality benefits of a larger MoE without paying the full training cost from scratch.

\paragraph{Growing model capacity during training.}
Progressive training grows networks mid-run to amortize compute: Net2Net introduced function-preserving width/depth transforms \citep{Chen2016Net2Net} and recent work extends these to Transformers via layer stacking~\citep{du2024stackingtransformers, agarwal2024stacking, bu2025dpt}; SPARKLING~\citep{yu2026sparkling} brings mid-training width expansion to MoE models.
All of these grow \emph{depth} or \emph{dense width}, raising active parameters and inference cost at each step.
A complementary line of work upcycles \emph{capacity} from an existing checkpoint: Sparse Upcycling converts a dense checkpoint into an MoE~\citep{komatsuzaki2022sparse}, with scaling laws derived in \citet{liew2025scalinglaws}; follow-ups improve initialization diversity~\citep{nakamura2025dropupcycling}, explore parameter-efficient variants~\citep{zhang2024bam, huang2025ders}, and study router design at scale~\citep{he2024upcyclingllm}.
These perform a dense$\,\to\,$MoE transition.
Closer to our setting, Nexus~\citep{gritsch2025nexus} expands an existing MoE by adaptively integrating new domain-specific experts.
Expert upcycling instead expands expert count \emph{within the MoE regime} by duplicating existing experts while holding top-$K$ fixed, increasing total parameters without increasing active parameters or inference cost.

\paragraph{Load balancing and saliency metrics.}
Imbalanced routing causes representation collapse in MoE, starving some experts of gradient signal~\citep{chi2022representationcollapse}; auxiliary balancing losses~\citep{shazeer2017outrageously, fedus2022switch} address this but trade off task performance, while loss-free load balancing~\citep{wang2024auxiliary} corrects imbalance via router-bias updates with no loss modification.
We adopt the loss-free scheme throughout so that every duplicated replica receives differentiated gradient signal, a prerequisite for symmetry breaking after upcycling.
Saliency metrics repurposed from pruning (weight magnitude~\citep{han2015learning}, second-order sensitivity~\citep{lecun1989optimal, hassibi1992second}, first-order Taylor approximations~\citep{molchanov2017taylor}, and Fisher-based estimators~\citep{soen2024tradeoffs, li2025fishersforfree}, with expert-level variants~\citep{lu2024experts}) identify which experts contribute most to the loss; we invert the pruning direction and use the same scores to choose which experts to duplicate, so pruning and upcycling act as duals (one recovers inference efficiency, the other expands training capacity).

\section{Expert upcycling}
\label{sec:method}
We introduce \textbf{expert upcycling}, a
capacity-expansion procedure that grows the number of
experts in an MoE model mid-training by reusing
learned parameters.
Given an $E$-expert MoE trained for $\tau$ gradient
steps, expert upcycling constructs an $mE$-expert
model by duplicating existing experts and extending
the router, formalized as the operator $U_m$ in
\S\ref{subsec:operator}, then continues training for
the remaining $T - \tau$ steps.
The conventional alternative, which we call
\emph{fixed-size training}, trains the $mE$-expert
model from random initialization for all $T$ steps.
Both approaches process the same number of training
tokens and produce a model with $mE$ experts and
identical per-token FLOPs since top-$K$ is fixed.

For expert upcycling to be a viable alternative to
fixed-size training, two conditions must hold:
(i)~it must be cheaper in total training compute, and
(ii)~it must close the quality gap to the from-scratch
model.
We address compute efficiency in
\S\ref{subsec:procedures} and quality gap closure in
\S\ref{subsec:operator}.

\subsection{Compute efficiency of expert upcycling}
\label{subsec:procedures}

The two approaches differ in training cost.
Let $s_E$ and $s_{mE}$ denote the per-step training
time of the $E$-expert and $mE$-expert models
respectively.
The larger model is more expensive per step due to
increased memory requirements, gradient and
optimizer-state updates over all expert parameters,
and all-to-all communication
overhead~\citep{du2024revisitingmoe, jin2025megascalemoe},
giving $s_E < s_{mE}$.
In our experimental setup (\S\ref{sec:setup}), we measure
$s_{mE} \approx 1.9 \times s_E$ when doubling expert
count (${\sim}2.2$s vs.\ ${\sim}4.2$s per step at 7B$\to$13B model scale).

The total training cost under each approach is
$\mathcal{C}^{\mathrm{fs}}$ for fixed-size training
and $\mathcal{C}^{\mathrm{up}}$ for expert upcycling:
\begin{align}
  \mathcal{C}^{\mathrm{fs}} &= T \times s_{mE}, \notag \\
  \mathcal{C}^{\mathrm{up}} &= \tau \times s_E + (T - \tau) \times s_{mE}.
\end{align}
Expert upcycling is strictly cheaper because the first
$\tau$ steps execute on the smaller model:
\begin{equation}
  \mathcal{C}^{\mathrm{fs}} - \mathcal{C}^{\mathrm{up}}
  \;=\;
  \tau \times (s_{mE} - s_E)
  \;>\; 0.
  \label{eq:flop_saving}
\end{equation}
The saving grows linearly with $\tau$ and the per-step
cost gap $s_{mE} - s_E$.
In our 7B-scale experiments, we find that
$\tau \approx \tfrac{2}{3}T$ is sufficient for the
upcycled model to match the from-scratch validation loss
(\S\ref{sec:results}), which translates to
${\sim}$32\% reduction in GPU hours.

\paragraph{Expanding existing models.}
When a trained $E$-expert checkpoint already exists, for instance from a prior training run or a public release, the
compute advantage is even larger: the pre-training
cost $\tau \times s_E$ is \emph{sunk}, and expert
upcycling requires only the incremental
$(T - \tau) \times s_{mE}$ for CPT on the expanded
model.
This makes expert upcycling particularly attractive
for continued pre-training of publicly available MoE
models.
In our 7B-scale experiments, the sunk-cost setting
at $\tau \approx \tfrac{2}{3}T$ reduces GPU hours by
${\sim}67\%$ compared to fixed-size training
(${\sim}50\%$ at $\tau = \tfrac{1}{2}T$).

\subsection{Quality gap closure and the upcycling operator}
\label{subsec:operator}

A cheaper procedure is only useful if it preserves
quality.  Adapting the OCO regret-telescoping analysis
of~\citet{bu2025dpt} for progressive training, we
decompose the gap between upcycling and fixed-size
training as a sum of two interpretable terms (formal
derivation in Appendix~\ref{app:proof}).

\begin{itemize}[nosep,leftmargin=*]
  \item \textbf{Capacity gap (Term~I):} during the
    first $\tau$ steps, training is in the $E$-expert
    model, a strictly less expressive class than the
    $mE$-expert target.  The gap between the two
    classes' optima is inherited by the upcycled model
    in proportion to $\tau/T$.
  \item \textbf{Initialization gain (Term~II):} at
    step $\tau$, the upcycled model must be initialized
    in the expanded parameter space.  The closer this
    initialization is to the $mE$-expert optimum
    (compared to fixed-size's random initialization),
    the larger upcycling's head start.
\end{itemize}

The compute saving $\tau(s_{mE} - s_E)$ from
\S\ref{subsec:procedures} and the capacity gap
(Term~I) pull $\tau$ in opposite directions: larger
$\tau$ saves more compute but widens Term~(I), while
smaller $\tau$ shrinks Term~(I) but erodes the compute
advantage.  We perform an ablation in
\S\ref{sec:results-budget} to identify a good
operating point for $\tau$.  Term~(II), in contrast,
is independent of $\tau$: it depends only on how the
expanded parameters are constructed, offering a
second, $\tau$-independent lever to close the quality
gap.  We therefore introduce an \emph{expert upcycling
operator} designed to maximize the initialization
gain by producing a warm start close to the
$mE$-expert optimum.

\begin{definition}[Expert Upcycling Operator]
\label{def:upcycling_operator}
Fix an integer expansion factor $m \ge 2$.
Given a trained $E$-expert parameter vector
$\theta_E \in \Theta_E$, the operator
$U_m \colon \Theta_E \to \Theta_{mE}$ constructs
$U_m(\theta_E)$ as follows:
\begin{enumerate}[nosep]
  \item \textbf{Expert replication.}
    Assign replication counts $\{r_e\}_{e=1}^{E}$
    with $r_e \ge 1$ and $\sum_e r_e = mE$;
    copy the parameters of expert~$e$ exactly $r_e$
    times.
    The canonical choice is \emph{uniform} replication
    ($r_e = m$ for all~$e$);
    non-uniform allocations are developed in
    \S\ref{sec:utility_upcycling}.
  \item \textbf{Router extension.}
    Copy the router weight vector of source expert~$e$
    to each of its $r_e$ replicas.
    Add independent noise
    $\epsilon \sim \mathcal{U}(-\delta,\delta)$
    (with $\delta \ll 1$) to the router \emph{biases}
    of replicated experts only, leaving the source
    experts' router parameters unchanged.
\end{enumerate}
\end{definition}

By construction, each new expert starts from a trained
weight vector and the router approximately reproduces
the pre-expansion routing distribution.
All other parameters (attention layers, embeddings,
and layer norms) are unchanged.
The post-expansion loss therefore satisfies
$\mathcal{L}_{mE}(U_m(\theta_E))
 \approx \mathcal{L}_E(\theta_E)$,
with a gap below $10^{-2}$ in practice (see \S\ref{sec:results}).
We refer to this property as \emph{warm initialization}:
$U_m$ places the expanded parameters substantially
closer to the $mE$-expert optimum than random
initialization, reducing Term~(II).\footnote{Exact
function preservation \`a la
Net2Net~\citep{Chen2016Net2Net} is structurally
unavailable here because top-$K$ routing is discrete;
expert upcycling therefore targets low initialization
loss directly rather than exact preservation.}

\subsection{Utility-based upcycling}
\label{sec:utility_upcycling}

Definition~\ref{def:upcycling_operator} leaves the
replication counts $\{r_e\}$ unspecified; uniform
replication ($r_e = m$ for all $e$) is the natural
baseline.  However, MoE experts are heterogeneous in
their contribution to the objective~\citep{lu2024experts},
so a non-uniform allocation that concentrates capacity
on high-importance experts should yield a warm start
closer to the $mE$-expert optimum, tightening Term~(II).
We allocate replicas using gradient-based importance
scores repurposed from the structured pruning
literature~\citep{lecun1989optimal, hassibi1992second,
molchanov2017taylor}.  We evaluate two first-order
utility scores, both computed from the gradient
$g_e = \nabla_{w_e}\mathcal{L}$ evaluated at
transition step~$\tau$:
\begin{itemize}[nosep,leftmargin=*]
  \item \textbf{Squared gradient norm:}
    $u_G(e) = \|g_e\|_2^2$, which captures how
    sensitive the loss is to each expert's parameters.
  \item \textbf{Weight--gradient saliency:}
    $u_{\mathrm{SAL}}(e)
     = \|w_e\|_2 \cdot \|g_e\|_2$, which combines
    parameter magnitude with gradient signal.
\end{itemize}
Both scores can be derived from a first-order Taylor
expansion of the loss
(Appendix~\ref{app:utility_theory}).
Replicas are allocated greedily to the
highest-scoring experts.
Both scores offer similar improvements over uniform
duplication, with squared gradient norm marginally
but consistently outperforming weight--gradient
saliency.
We also tried curvature-normalized variants
($\|g_e\|_2^2 / H_e$) and weight-norm-only scoring
($\|w_e\|_2^2$); neither worked as well
(Appendix~\ref{app:utility_theory}).

Beyond utility-based selection, we evaluated a broader
family of diversity-inducing heuristics at the expert
and router level (noise injection, orthogonalization,
drop-based re-initialization, interpolation, and
others; 20 variants in total, detailed in
Appendix~\ref{app:heuristic_upcycling_methods}).  None
exceed simple copy-paste duplication by more than
$10^{-3}$ in validation loss, indicating that
warm-start quality in this setting is dominated by
which experts are replicated rather than how the
copied parameters are perturbed.

\subsection{Post-expansion dynamics}
\label{subsec:dynamics}

Immediately after applying $U_m$, the replicated
experts are near-identical copies.
Three mechanisms break this symmetry during CPT:
the router bias perturbation from $U_m$ creates
initial routing asymmetry; loss-free load
balancing~\citep{wang2024auxiliary} ensures every
replica receives gradient signal; and stochastic
gradient diversity drives a self-reinforcing cycle
of specialization (different parameters $\to$
different routing $\to$ different gradients).

\section{Experimental setup}

\label{sec:setup}

\paragraph{Architecture and training.}
Our main result uses a 20-layer interleaved MoE,
which alternates dense and MoE layers as in
Llama~4~\citep{metalllama4}, with
${\sim}$7B$\to$13B total and ${\sim}$1B active
non-embedding parameters.
We focus on the interleaved architecture for the
majority of our experiments, including both recipe
ablations and the 7B-scale run, as only half the
layers incur all-to-all communication, substantially
reducing per-step training
time~\citep{du2024revisitingmoe, deepseekai2024deepseekv3}
and allowing faster iteration on an NVIDIA A100 GPU
cluster.
Most ablations for deriving the practical upcycling
recipe are conducted at the ${\sim}$1B total parameter
scale on the same interleaved architecture.
To assess generalizability to full MoE, we conduct
an additional ablation at the ${\sim}$1B scale on a
full MoE architecture with 256 experts and top-$K{=}8$,
matching the routing configuration of frontier MoE
models~\citep{deepseekai2024deepseekv3, glm45, kimik2}.
All models use top-$K$ gating with $K \in \{2, 8\}$ and no shared experts, Grouped Query Attention (GQA), and RoPE positional embeddings.
Full model architecture details and optimal training hyperparameters, as determined by preliminary scaling sweeps, are provided in Tables~\ref{tab:model_configs_interleaved}--\ref{tab:model_configs_fullmoe} (Appendix~\ref{app:model_configs}).
Optimization follows a Warmup--Stable--Decay (WSD) schedule with loss-free load balancing~\citep{wang2024auxiliary}. 
Training is performed using data parallelism and tensor parallelism.

\paragraph{Data.}
To separate the effects of pre-training from continued pre-training (CPT), we use disjoint data splits: the base models are trained on the pre-training split, and all CPT stages are performed on a separate CPT split, thereby avoiding data leakage between stages. Small-scale ablation experiments use DCLM~\citep{li2024datacomplm}. The 7B-scale main experiment uses a curated data mixture emphasizing instruction following, logical reasoning, and math.

\paragraph{Evaluation protocol.}
For each experiment, we compare three configurations:
a fixed $E$-expert model with no expansion, our upcycled
$E{\to}mE$ model, and a fixed $mE$-expert model
trained from scratch, all at matched total token
budget. Each training stage concludes with a 10\%
annealing phase.
For the 7B-total/1B-active main result, we report both
downstream benchmark accuracy (11 benchmarks) and
validation loss. For the ${\sim}$1B-total/${\sim}$144M-active-non-embedding
ablation experiments, we report validation loss only,
as most downstream benchmarks do not reliably
differentiate models at this
scale~\citep{wei2022emergent}.
To compare across settings, we report \emph{upcycling
efficiency} over validation loss $L$:
\begin{equation}
  \eta = \frac{L(\text{Fixed-}E) - L(\text{Upcycled})}
             {L(\text{Fixed-}E) - L(\text{Fixed-}mE)},
\end{equation}
which measures normalized gap closure;
a value of $1$ indicates complete gap closure.

\paragraph{Upcycling procedure.}
At transition step $\tau$, optimizer states are reset (matching sparse upcycling~\citep{komatsuzaki2022sparse}), and utility scores are computed from gradients averaged over 10 batches. Replicas are allocated greedily subject to a per-expert duplicate cap of $n{=}3$, preventing any single expert from monopolizing the allocation. The router-bias noise in Definition~\ref{def:upcycling_operator} uses $\delta = 10^{-3}$.

\section{Results}
\label{sec:results}

We first demonstrate expert upcycling at the full 7B$\to$13B scale (\S\ref{sec:results-7b}). The design choices for this experiment (operator, CPT budget, and transition timing) were informed by systematic ablations at the $\sim$1B scale (\S\ref{subsec:recipe-ablations}), which study each mechanism of upcycling and yield a practical recipe.

\subsection{Expert upcycling at scale}

\label{sec:results-7b}

\begin{table*}[t]
\centering
\caption{Expert upcycling at scale. A 20-layer interleaved MoE ($\sim$7B total / $\sim$1B active non-embedding parameters) is pre-trained with 32 experts on 380B tokens, then continued for 50\% (190B) or 100\% (380B) of the pre-training budget. \textbf{Fixed-32}: 32-expert model continued without expansion. \textbf{Upcycled 32$\to$64 (Ours)}: expand to 64 experts via duplication, then CPT. \textbf{Fixed-64}: 64-expert model trained from scratch. Best value per row in \textbf{bold}. Accuracy higher$\uparrow$; Val.\ Loss lower$\downarrow$.}
\label{tab:main_highlight}
\small
\begin{tblr}{
  colspec  = {l ccc ccc},
  hline{1} = {1.5pt, solid},
  hline{Z} = {1.5pt, solid},
  hline{2} = {2-4}{0.4pt, solid},
  hline{2} = {5-7}{0.4pt, solid},
  hline{3} = {0.8pt, solid},
  hline{4} = {0.4pt, dashed},
  hline{15} = {0.4pt, solid},
  cell{1}{2} = {c=3}{c, font=\bfseries},
  cell{1}{5} = {c=3}{c, font=\bfseries},
  cell{2}{1-7} = {font=\bfseries},
  cell{3}{3} = {bg=gray!15},
  cell{3}{6} = {bg=gray!15},
  column{3,6} = {bg=gray!15},
}
  & 50\% CPT & & & 100\% CPT & & \\
  Metric & Fixed-32 & Upcycled (Ours) & Fixed-64 & Fixed-32 & Upcycled (Ours) & Fixed-64 \\
  Val.\ Loss ($\downarrow$) & 1.339 & \textbf{1.305} & 1.308 & 1.301 & \textbf{1.263} & 1.267 \\
  MMLU        & 43.9 & 46.2          & \textbf{47.4} & 47.5 & 52.3          & \textbf{52.7} \\
  BBH CoT     & 19.4 & 28.0          & \textbf{33.8} & 33.4 & 43.8          & \textbf{45.0} \\
  GSM8K       & 28.1 & 36.0          & \textbf{40.1} & 39.3 & 48.3          & \textbf{49.1} \\
  IFEval      & 19.6 & 24.4          & \textbf{27.5} & 20.7 & 27.6          & \textbf{29.4} \\
  HellaSwag   & 62.4 & \textbf{65.0} & 63.9          & 65.1 & \textbf{67.3} & 66.1 \\
  ARC-Challenge & 45.6 & 46.2        & \textbf{47.3} & 47.5 & 48.7          & \textbf{48.8} \\
  ARC-Easy    & 68.8 & 72.9          & \textbf{74.4} & 74.4 & \textbf{76.0} & 75.3 \\
  PIQA        & 74.2 & \textbf{75.9} & 75.5          & 76.6 & \textbf{77.4} & 77.5 \\
  OpenBookQA  & 36.0 & \textbf{39.4} & 38.4          & 37.8 & 38.8          & \textbf{39.8} \\
  SciQ        & 91.7 & 93.0          & \textbf{93.7} & 93.3 & 94.0          & \textbf{94.1} \\
  Social IQA  & 43.6 & \textbf{46.3} & 45.5          & 46.5 & \textbf{46.5} & 46.1 \\
  \textbf{Avg (acc $\uparrow$)} & 48.5 & 52.1 & \textbf{53.4} & 52.9 & 56.4 & \textbf{56.7} \\
\end{tblr}
\end{table*}

Table~\ref{tab:main_highlight} reports validation loss and downstream benchmark accuracy across 11 tasks for Fixed-32, the upcycled 32$\to$64 model (gradient-norm guided duplication, \S\ref{sec:results-utility}), and Fixed-64 at 50\% and 100\% CPT.
At 50\% CPT, the upcycled model surpasses Fixed-64
on validation loss (efficiency 109.7\%); average
downstream accuracy across 11 benchmarks is within
1.3 points of Fixed-64 (52.1 vs.\ 53.4), saving
${\sim}$32\% of GPU-hours.
Commonsense and language understanding tasks
(HellaSwag, PIQA, Social~IQA, OpenBookQA)
match or exceed Fixed-64 at this point.
The remaining gap concentrates in knowledge and
reasoning tasks (MMLU, BBH, GSM8K, IFEval), which
continue to improve with additional CPT.
By 100\% CPT, these tasks largely converge as well
(e.g., BBH: 43.8 vs.\ 45.0; GSM8K: 48.3 vs.\ 49.1),
bringing average accuracy to 56.4 vs.\ 56.7 and
validation loss to 1.263 vs.\ 1.267, with efficiency
111.8\%.

Immediately after upcycling, the 64-expert model's training loss on the CPT split is 1.38, close to the 32-expert source at 1.32 and far below the 10.5 of a randomly initialized 64-expert model; within the first 22B CPT tokens ($\sim$6\% of pre-training) the upcycled loss falls below Fixed-32 and subsequent CPT closes the gap to Fixed-64. Doubling the CPT budget from 50\% to 100\% raises average accuracy from 52.1 to 56.4, consistent with the capacity-gap penalty (Term~I) shrinking as CPT length grows.

\paragraph{Generalization to full MoE.}
Expert upcycling also transfers beyond the interleaved architecture: on a full MoE with 256 experts and top-$K{=}8$ ($\sim$3\% activation ratio), $\|g\|^2$ utility-based upcycling achieves $\geq$93\% gap-closure efficiency across 154M--1B total parameters (Appendix~\ref{app:fullmoe_results}, Table~\ref{tab:fullmoe_results}).

\subsection{Recipe ablations}
\label{subsec:recipe-ablations}

Three ablations at the $\sim$1B scale test the two-term decomposition of \S\ref{subsec:operator} and yield the recipe used in the 7B$\to$13B run. Training-budget allocation (\S\ref{sec:results-budget}) trades Term~(I) against Term~(II) via transition timing and CPT length. Operator design (\S\ref{sec:results-utility}) targets Term~(II) along two axes: which experts to duplicate, and how to initialize the duplicates. The activation-ratio sweep (\S\ref{sec:results-sparsity}) stress-tests Term~(I) by widening the source--target capacity gap, contrasting expert upcycling with sparse (dense$\to$MoE) upcycling. All three share a 10-layer interleaved MoE (32$\to$64 experts, $\sim$1B total / 144M active non-embedding; Table~\ref{tab:model_configs_interleaved}).

\subsubsection{Training budget allocation: pre-training, CPT, and transition timing}
\label{sec:results-budget}

Term~(I) predicts that gap closure is governed by how much training is spent in the expanded model relative to the smaller one.
This motivates two practical questions: \emph{when training from scratch, at what point should the transition occur?} And \emph{given an already pre-trained model, how much CPT is needed after upcycling?}
We test these with two experiments (Table~\ref{tab:budget_highlight}), using the ${\sim}$1B-total/144M-active 10-layer interleaved MoE pre-trained for ${\sim}$50K steps.

\begin{table*}[t]
\centering
\caption{Training budget allocation for expert upcycling (10-layer interleaved MoE, 32$\to$64 experts, top-$K{=}2$, pre-trained for $\tau{=}50$K steps; upcycled model uses $\|g\|^2$ utility-based duplication). \textbf{(a)} When to upcycle when training from scratch ($T{=}100$K steps fixed, transition point $\tau$ varies). \textbf{(b)} How much CPT is needed after upcycling an already pre-trained model ($\tau$ fixed, $T$ varies; $\tau/T$ is the fraction of total training spent in the smaller model).}
\label{tab:budget_highlight}
\small
\setlength{\tabcolsep}{4pt}
\begin{tblr}{
  colspec = {rr rrrQ[r,wd=0.8cm] rr rrrQ[r,wd=0.8cm]},
  hline{1} = {1.5pt, solid},
  hline{Z} = {1.5pt, solid},
  hline{2} = {1-6}{0.4pt, solid},
  hline{2} = {7-12}{0.4pt, solid},
  hline{3} = {0.8pt, solid},
  cell{1}{1} = {c=6}{c},
  cell{1}{7} = {c=6}{c},
  cell{2}{1-12} = {font=\bfseries},
  row{1-2} = {bg=white},
  column{4} = {bg=gray!15},
  column{10} = {bg=gray!15},
}
\textit{(a) When to upcycle? ($T{=}100$K fixed, $\tau$ varies)} & & & & & &
\textit{(b) CPT budget sweep ($\tau{=}50$K fixed, $T$ varies)} & & & & & \\
\makecell{$\tau$\\steps} & \makecell{$\tau/T$} & \makecell{Fixed\\32} & \makecell{Up-\\cycled} & \makecell{Fixed\\64} & Eff.(\%) &
\makecell{CPT\\steps} & \makecell{$\tau/T$} & \makecell{Fixed\\32} & \makecell{Up-\\cycled} & \makecell{Fixed\\64} & Eff.(\%) \\
5K  & 0.05 & 2.804 & 2.753 & 2.741 & 81.0 & 5K  & 0.91 & 2.850 & 2.833 & 2.801 & 34.7 \\
13K & 0.13 & 2.805 & 2.754 & 2.754 & 100.0 & 13K & 0.80 & 2.835 & 2.803 & 2.785 & 64.0 \\
25K & 0.25 & 2.806 & 2.754 & 2.754 & 100.0 & 25K & 0.67 & 2.823 & 2.780 & 2.772 & 84.3 \\
38K & 0.38 & 2.806 & 2.757 & 2.757 & 100.0 & 38K & 0.57 & 2.815 & 2.769 & 2.763 & 88.5 \\
51K & 0.51 & 2.809 & 2.759 & 2.758 & 98.0  & 51K & 0.50 & 2.809 & 2.759 & 2.758 & 98.0 \\
\end{tblr}
\end{table*}

\paragraph{Transition timing.}
Under a fixed total budget of 100K steps (Table~\ref{tab:budget_highlight}a), upcycling early ($\tau/T \leq 0.25$) achieves near-complete gap closure (94--100\% efficiency). Very early transitions ($\tau/T = 0.05$) underperform slightly, likely because the source model has seen too few tokens for experts to develop meaningful specialization, weakening the warm initialization that Term~(II) relies on.

\paragraph{CPT budget.}
Sweeping CPT from 10\% to 100\% of the pre-training budget (Table~\ref{tab:budget_highlight}b), efficiency rises monotonically from 34.7\% to 98.0\%. At least 50\% CPT is needed for strong gap closure, consistent with Term~(I): duplicated experts require sufficient post-upcycling optimization to break symmetry and specialize. Together, these results confirm that CPT budget is the binding constraint: pre-training determines initialization quality, while CPT determines the extent of expert differentiation.

\subsubsection{Expert upcycling strategies}
\label{sec:results-utility}

\paragraph{Which experts to duplicate?}
The 7B$\to$13B result used gradient-norm guided duplication; here we study how much the choice of duplication strategy matters.
Since experts in a trained MoE contribute unevenly to the objective, non-uniform duplication that concentrates capacity on high-importance experts may yield a better initialization than simply copying every expert once.
We compare four utility-based expert duplication strategies from \S~\ref{sec:utility_upcycling} (weight norm $u_{\text{WN}}$, gradient norm $u_{\text{GN}}$, saliency $u_{\text{SAL}}$, and curvature-normalized utility $u_{\text{CN}}$) against two baselines: uniform duplication and random initialization~\citep{he2015kaiming}, across CPT budgets, with experts ranked per layer and duplicated via greedy selection with replacement. Table~\ref{tab:duplication_comparison} summarizes results.

\begin{table*}[t]
\centering
\caption{Comparison of duplication strategies on the 10-layer interleaved MoE (32$\to$64 experts, top-$K{=}2$, $\sim$1B total / 144M active non-embedding parameters; upcycling happens at $\tau{=}50$K steps for all rows). Rows correspond to different CPT budgets. \textbf{Fixed-32} and \textbf{Fixed-64} bracket the achievable range. \textbf{Random} initializes the new 32 experts with Kaiming initialization~\citep{he2015kaiming}. \textbf{Uniform} copies every expert once. The remaining four columns are utility-based strategies that preferentially duplicate high-importance experts: weight norm ($\|w\|^2$), saliency ($\|g\|\cdot\|w\|$), gradient norm ($\|g\|^2$), and curvature-normalized ($\|g\|^2/H$). All values are validation loss ($\downarrow$).}
\label{tab:duplication_comparison}
\small
\setlength{\tabcolsep}{4pt}
\begin{tblr}{
  colspec  = {l cc cc cccc},
  hline{1} = {1.5pt, solid},
  hline{Z} = {1.5pt, solid},
  hline{2} = {4-9}{0.4pt, solid},
  hline{3} = {2-3}{0.4pt, solid},
  hline{3} = {4-9}{0.4pt, solid},
  hline{4} = {0.8pt, solid},
  cell{1}{2} = {c=2}{c, font=\bfseries},
  cell{1}{4} = {c=6}{c, font=\bfseries},
  cell{2}{4} = {c=2}{c},
  cell{2}{6} = {c=4}{c},
  cell{3}{1-9} = {font=\bfseries},
  row{1-3} = {bg=white},
  column{8} = {bg=gray!15},
}
  & Fixed & & Expert Upcycling & & & & & \\
  & & & Non-utility & & Utility-based & & & \\
  CPT steps & Fixed-32 & Fixed-64 & Random & Uniform & $\|w\|^2$ & $\|g\|\cdot\|w\|$ & $\|g\|^2$ & $\|g\|^2/H$ \\
  13K  & 2.857 & 2.808 & 3.107 & 2.853 & 2.846 & 2.846 & 2.844 & 2.845 \\
  25K  & 2.835 & 2.785 & 3.010 & 2.809 & 2.807 & 2.809 & 2.804 & 2.805 \\
  38K  & 2.821 & 2.769 & 2.969 & 2.787 & 2.785 & 2.778 & 2.773 & 2.776 \\
  50K  & 2.809 & 2.758 & 2.963 & 2.769 & 2.771 & 2.766 & 2.759 & 2.768 \\
\end{tblr}
\end{table*}

Selective duplication consistently outperforms both baselines at every CPT budget. Random initialization performs far worse than even the Fixed-32 baseline (e.g., loss 3.107 vs.\ 2.857 at 25\% CPT), confirming that warm initialization is essential. Among warm-start strategies, utility-based selection outperforms uniform duplication at every CPT budget, with the largest gains when CPT is limited, more than tripling gap closure at 25\% CPT (26.5\% vs.\ 8.2\%, computed from Table~\ref{tab:duplication_comparison}). The advantage narrows at higher budgets as longer training partially compensates for initialization differences. Among the four utility strategies, gradient norm ($\|g\|^2$) performs best overall; curvature-normalized ($\|g\|^2/H$) and saliency ($\|g\|\cdot\|w\|$) are close behind, and weight-norm-only ($\|w\|^2$) lags slightly. In practice, $\|g\|^2$ is the recommended default.

\paragraph{Initialization diversity does not substitute for initialization quality.}
We also evaluated 20 diversity-inducing initialization heuristics (noise injection, drop upcycling~\citep{nakamura2025dropupcycling}, interpolation, orthogonalization, SVD-based perturbation, sparse-code mixing; 10 expert-level and 10 router-level variants); none exceed simple copy-paste duplication by more than $10^{-3}$ in validation loss. Across all $n{=}65$ utility + heuristic runs, the Spearman rank correlation between initialization and terminal loss is $\rho{=}0.80$ (val) / $0.86$ (train): runs that start worse end worse.

\subsubsection{Effect of activation ratio and comparison with sparse upcycling}
\label{sec:results-sparsity}

Term~(I) penalizes the capacity gap
$\mathcal{L}_E^\star - \mathcal{L}_{mE}^\star$ between source and target.
We study this on the 8-layer interleaved MoE with top-$K{=}1$ and uniform
duplication, comparing expert upcycling (MoE$\,\to\,$MoE) against sparse
upcycling~\citep{komatsuzaki2022sparse} (dense$\,\to\,$MoE) across target
activation ratios from 25\% down to 3.13\%
(Table~\ref{tab:sparsity_highlight}).
For each target, expert upcycling starts from an MoE base with half the
target expert count, while sparse upcycling starts from a dense checkpoint.

Expert upcycling consistently produces losses close to the Fixed-$mE$
ceiling across all activation ratios, though the residual gap grows
at lower ratios (0.005 at 25\% vs.\ 0.020 at 3.13\%, computed from Table~\ref{tab:sparsity_highlight}).
Sparse upcycling, by contrast, fails to match even the Fixed-$E$ baseline
in every setting: the dense$\,\to\,$MoE transition spans too large a
capacity gap for CPT to close, confirming the Term~(I) prediction.
The gap between the two methods widens as the target activation ratio
decreases, from 0.026 at 25\% to 0.241 at 3.13\%.
Stabilized sparse-upcycling variants~\citep{nakamura2025dropupcycling, jiang2025improved} could narrow this gap and are future work.

\begin{table}[t]
\centering
\caption{Effect of activation ratio and comparison with sparse upcycling (8-layer interleaved MoE, top-$K{=}1$, uniform duplication). Expert upcycling (Ours) starts from an MoE base with half the target expert count; sparse upcycling starts from a dense checkpoint. Both methods target the same $mE$-expert model (Fixed-$mE$). All values are validation loss ($\downarrow$).}
\label{tab:sparsity_highlight}
\small
\setlength{\tabcolsep}{4pt}
\begin{tblr}{
  colspec = {ccccc},
  hline{1} = {1.5pt, solid},
  hline{Z} = {1.5pt, solid},
  hline{2} = {0.8pt, solid},
  row{1} = {bg=white, font=\bfseries},
  column{4} = {bg=gray!15},
}
  \makecell{Target\\$K/E$} & Fixed-$E$ & Fixed-$mE$ & \makecell{Ours\\(MoE$\to$MoE)} & \makecell{Sparse Upc.\\(Dense$\to$MoE)} \\
  25\%   & 3.085 & 3.056 & 3.061 & 3.087 \\
  12.5\% & 3.056 & 3.018 & 3.025 & 3.086 \\
  6.25\% & 3.018 & 2.986 & 2.992 & 3.069 \\
  3.13\% & 2.894 & 2.788 & 2.808 & 3.049 \\
\end{tblr}
\end{table}

\section{Discussion}
\label{sec:discussion}
\label{sec:limitations}

We introduced expert upcycling, which progressively expands MoE capacity by duplicating experts and extending the router mid-training while holding top-$K$ routing fixed to preserve inference cost. A theoretical decomposition into a capacity gap and an initialization gain guided the operator design and ablations; in our 7B$\to$13B experiments the upcycled model achieves lower validation loss than the fixed-size baseline while saving 24--32\% of GPU hours (50--67\% when the source model's pre-training cost is sunk); at 100\% CPT, the downstream accuracy gap closes to within 0.3 points (average across 11 benchmarks). The activation-ratio sweep (\S\ref{sec:results-sparsity}) further suggests iterated upcycling, repeatedly doubling expert count through successive steps to keep Term~(I) small, as a natural extension.

Our results cover $m{=}2$ upcycling on MoE architectures up to 7B parameters, with DCLM for ablations and an English-majority mixture for the 7B run. Open directions include larger $m$, frontier-scale ($>$10B) models, and multilingual or distribution-shifted CPT.

{
\bibliographystyle{plainnat}
\bibliography{ref}

@inproceedings{he2015kaiming,
  title     = {Delving Deep into Rectifiers: Surpassing Human-Level Performance on {ImageNet} Classification},
  author    = {He, Kaiming and Zhang, Xiangyu and Ren, Shaoqing and Sun, Jian},
  booktitle = {Proceedings of the IEEE International Conference on Computer Vision (ICCV)},
  year      = {2015}
}

@inproceedings{han2015learning,
  title     = {Learning both Weights and Connections for Efficient Neural Networks},
  author    = {Han, Song and Pool, Jeff and Tran, John and Dally, William J.},
  booktitle = {Advances in Neural Information Processing Systems (NeurIPS)},
  volume    = {28},
  year      = {2015}
}

@inproceedings{ludziejewski2025jointmoe,
  title        = {Joint {MoE} Scaling Laws: Mixture of Experts Can Be Memory Efficient},
  author       = {Jan Ludziejewski and Maciej P{\'i}oro and Jakub Krajewski and
                  Maciej Stefaniak and Micha{\l} Krutul and Jan Ma{\l}a\'{s}nicki and
                  Marek Cygan and Piotr Sankowski and Kamil Adamczewski and
                  Piotr Mi{\l}o\'{s} and Sebastian Jaszczur},
  booktitle    = {Proceedings of the 42nd International Conference on Machine Learning (ICML)},
  pages        = {41056--41073},
  year         = {2025},
  series       = {Proceedings of Machine Learning Research},
  volume       = {267},
  publisher    = {PMLR},
  url          = {https://proceedings.mlr.press/v267/ludziejewski25a.html},
  doi          = {10.48550/ARXIV.2502.05172},
  note         = {Also available as arXiv:2502.05172},
}

@article{shazeer2017outrageously,
  title   = {Outrageously Large Neural Networks: The Sparsely-Gated Mixture-of-Experts Layer},
  author  = {Shazeer, Noam and Mirhoseini, Azalia and Maziarz, Krzysztof and Davis, Andy and Le, Quoc and Hinton, Geoffrey and Dean, Jeff},
  journal = {arXiv preprint arXiv:1701.06538},
  year    = {2017}
}

@article{komatsuzaki2022sparse,
  title   = {Sparse Upcycling: Training Mixture-of-Experts from Dense Checkpoints},
  author  = {Komatsuzaki, Aran and Puigcerver, Joan and Lee-Thorp, James and Riquelme Ruiz, Carlos and Mustafa, Basil and Ainslie, Joshua and Tay, Yi and Dehghani, Mostafa and Houlsby, Neil},
  journal = {arXiv preprint arXiv:2212.05055},
  year    = {2022}
}

@article{nakamura2025dropupcycling,
  title   = {Drop-Upcycling: Training Sparse Mixture of Experts with Partial Re-initialization},
  author  = {Nakamura, Taishi and Akiba, Takuya and Fujii, Kazuki and Oda, Yusuke and Yokota, Rio and Suzuki, Jun},
  journal = {arXiv preprint arXiv:2502.19261},
  year    = {2025}
}

@article{krajewski2024finegrained,
  title   = {Scaling Laws for Fine-Grained Mixture of Experts},
  author  = {Krajewski, Jakub and Ludziejewski, Jan and Adamczewski, Kamil and Pi{\'o}ro, Maciej and Krutul, Micha{\l} and Antoniak, Szymon and Ciebiera, Kamil and Kr{\'o}l, Krystian and Odrzyg{\'o}{\'z}d{\'z}, Tomasz and Sankowski, Piotr and Cygan, Marek and Jaszczur, Sebastian},
  journal = {arXiv preprint arXiv:2402.07871},
  year    = {2024}
}

@article{tian2025greaterleverage,
  title   = {Towards Greater Leverage: Scaling Laws for Efficient Mixture-of-Experts Language Models},
  author  = {Tian, Changxin and Chen, Kunlong and Liu, Jia and Liu, Ziqi and Zhang, Zhiqiang and Zhou, Jun},
  journal = {arXiv preprint arXiv:2507.17702},
  year    = {2025}
}

@article{lu2024experts,
  title   = {Not All Experts are Equal: Efficient Expert Pruning and Skipping for Mixture-of-Experts Large Language Models},
  author  = {Lu, Xudong and Liu, Qi and Xu, Yuhui and Zhou, Aojun and Huang, Siyuan and Zhang, Bo and Yan, Junchi and Li, Hongsheng},
  journal = {arXiv preprint arXiv:2402.14800},
  year    = {2024}
}

@inproceedings{li2025fishersforfree,
  title     = {Fishers for Free? Approximating the Fisher Information Matrix by Recycling the Squared Gradient Accumulator},
  author    = {Li, Yu Xin and Dangel, Felix and Tam, Derek and Raffel, Colin},
  booktitle = {Proceedings of the 42nd International Conference on Machine Learning (ICML)},
  pages     = {34252--34270},
  year      = {2025},
  volume    = {267},
  series    = {Proceedings of Machine Learning Research},
  publisher = {PMLR}
}

@inproceedings{Chen2016Net2Net,
  title     = {Net2Net: Accelerating Learning via Knowledge Transfer},
  author    = {Chen, Tianqi and Goodfellow, Ian and Shlens, Jonathon},
  booktitle = {International Conference on Learning Representations (ICLR)},
  year      = {2016},
  url       = {https://arxiv.org/abs/1511.05641}
}

@inproceedings{Lepikhin2020GShard,
  title     = {GShard: Scaling Giant Models with Conditional Computation and Automatic Sharding},
  author    = {Lepikhin, Dmitry and Lee, HyoukJoong and Xu, Yuanzhong and Chen, Dehao and Firat, Orhan and Huang, Yanping and Krikun, Maxim and Shazeer, Noam and Chen, Zhifeng},
  booktitle = {International Conference on Learning Representations (ICLR)},
  year      = {2021},
  url       = {https://arxiv.org/abs/2006.16668}
}

@article{fedus2022switch,
  title   = {Switch Transformers: Scaling to Trillion Parameter Models with Simple and Efficient Sparsity},
  author  = {Fedus, William and Zoph, Barret and Shazeer, Noam},
  journal = {Journal of Machine Learning Research (JMLR)},
  volume  = {23},
  number  = {120},
  pages   = {1--39},
  year    = {2022},
  url     = {https://jmlr.org/papers/v23/21-0998.html}
}

@inproceedings{hassibi1992second,
  title={Second order derivatives for network pruning: Optimal brain surgeon},
  author={Hassibi, Babak and Stork, David G},
  booktitle={Advances in Neural Information Processing Systems (NeurIPS)},
  volume={5},
  pages={164--171},
  year={1992}
}

@inproceedings{lecun1989optimal,
  title={Optimal brain damage},
  author={LeCun, Yann and Denker, John and Solla, Sara},
  booktitle={Advances in Neural Information Processing Systems (NeurIPS)},
  volume={2},
  pages={598--605},
  year={1989}
}

@article{wang2024auxiliary,
  title   = {Auxiliary-Loss-Free Load Balancing Strategy for Mixture-of-Experts},
  author  = {Wang, Lean and Gao, Huazuo and Zhao, Chenggang and Sun, Xu and Dai, Damai},
  journal = {arXiv preprint arXiv:2408.15664},
  year    = {2024}
}

@misc{bu2025dpt,
  title         = {Deep Progressive Training: scaling up depth capacity of zero/one-layer models},
  author        = {Bu, Zhiqi},
  year          = {2025},
  eprint        = {2511.04981},
  archivePrefix = {arXiv},
  primaryClass  = {cs.LG},
  url           = {https://arxiv.org/abs/2511.04981}
}

@misc{agarwal2024stacking,
  title         = {Stacking as Accelerated Gradient Descent},
  author        = {Agarwal, Naman and Awasthi, Pranjal and Kale, Satyen and Zhao, Eric},
  year          = {2024},
  eprint        = {2403.04978},
  archivePrefix = {arXiv},
  primaryClass  = {cs.LG},
  url           = {https://arxiv.org/abs/2403.04978}
}

@inproceedings{molchanov2017taylor,
  title={Pruning Convolutional Neural Networks for Resource Efficient Inference},
  author={Molchanov, Pavlo and Tyree, Stephen and Karras, Tero and Aila, Timo and Kautz, Jan},
  booktitle={International Conference on Learning Representations (ICLR)},
  year={2017}
}

@inproceedings{soen2024tradeoffs,
  title     = {Trade-Offs of Diagonal Fisher Information Matrix Estimators},
  author    = {Alexander Soen and Ke Sun},
  booktitle = {Advances in Neural Information Processing Systems (NeurIPS)},
  year      = {2024},
  eprint    = {2402.05379},
  archivePrefix = {arXiv},
  primaryClass  = {cs.LG}
}

@article{raposo2024mixtureofdepths,
  title   = {Mixture-of-Depths: Dynamically Allocating Compute in Transformer-Based Language Models},
  author  = {Raposo, David and Ritter, Sam and Richards, Blake and Lillicrap, Timothy and Humphreys, Peter Conway and Santoro, Adam},
  journal = {arXiv preprint arXiv:2404.02258},
  year    = {2024},
}

@article{kirkpatrick2017ewc,
  title   = {Overcoming Catastrophic Forgetting in Neural Networks},
  author  = {Kirkpatrick, James and Pascanu, Razvan and Rabinowitz, Neil and Veness, Joel and Desjardins, Guillaume and Rusu, Andrei A and Milan, Kieran and Quan, John and Ramalho, Tiago and Grabska-Barwinska, Agnieszka and others},
  journal = {Proceedings of the National Academy of Sciences},
  volume  = {114},
  number  = {13},
  pages   = {3521--3526},
  year    = {2017},
}

@article{gupta2023continual,
  title   = {Continual Pre-Training of Large Language Models: How to (re)warm your model?},
  author  = {Gupta, Kshitij and Th{\'e}rien, Benjamin and Ibrahim, Adam and Richter, Mats L. and Anthony, Quentin and Belilovsky, Eugene and Rish, Irina and Lesort, Timoth{\'e}e},
  journal = {arXiv preprint arXiv:2308.04014},
  year    = {2023}
}

@article{glm45,
  title   = {{GLM-4.5}: Agentic, Reasoning, and Coding ({ARC}) Foundation Models},
  author  = {{GLM-4.5 Team}},
  journal = {arXiv preprint arXiv:2508.06471},
  year    = {2025},
  url     = {https://arxiv.org/abs/2508.06471},
}

@article{jiang2024mixtral,
  title={Mixtral of experts},
  author={Jiang, Albert Q and Sablayrolles, Alexandre and Roux, Antoine and Mensch, Arthur and Savary, Blanche and Bamford, Chris and Chaplot, Devendra Singh and Casas, Diego de las and Hanna, Emma Bou and Bressand, Florian and others},
  journal={arXiv preprint arXiv:2401.04088},
  year={2024}
}

@misc{li2024datacomplm,
  title         = {DataComp-LM: In search of the next generation of training sets for language models},
  author        = {Jeffrey Li and Alex Fang and Georgios Smyrnis and Maor Ivgi and Matt Jordan and Samir Gadre and Hritik Bansal and Etash Guha and Sedrick Keh and Kushal Arora and Saurabh Garg and Rui Xin and Niklas Muennighoff and Reinhard Heckel and Jean Mercat and Mayee Chen and Suchin Gururangan and Mitchell Wortsman and Alon Albalak and Yonatan Bitton and Marianna Nezhurina and Amro Abbas and Cheng-Yu Hsieh and Dhruba Ghosh and Josh Gardner and Maciej Kilian and Hanlin Zhang and Rulin Shao and Sarah Pratt and Sunny Sanyal and Gabriel Ilharco and Giannis Daras and Kalyani Marathe and Aaron Gokaslan and Jieyu Zhang and Khyathi Chandu and Thao Nguyen and Igor Vasiljevic and Sham Kakade and Shuran Song and Sujay Sanghavi and Fartash Faghri and Sewoong Oh and Luke Zettlemoyer and Kyle Lo and Alaaeldin El-Nouby and Hadi Pouransari and Alexander Toshev and Stephanie Wang and Dirk Groeneveld and Luca Soldaini and Pang Wei Koh and Jenia Jitsev and Thomas Kollar and Alexandros G. Dimakis and Yair Carmon and Achal Dave and Ludwig Schmidt and Vaishaal Shankar},
  year          = {2024},
  eprint        = {2406.11794},
  archivePrefix = {arXiv},
  primaryClass  = {cs.LG},
  url           = {https://arxiv.org/abs/2406.11794}
}

@article{du2022glam,
  title   = {{GLaM}: Efficient Scaling of Language Models with Mixture-of-Experts},
  author  = {Du, Nan and Huang, Yanping and Dai, Andrew M. and Tong, Simon and Lepikhin, Dmitry and Xu, Yuanzhong and Krikun, Maxim and Zhou, Yanqi and Yu, Adams Wei and Firat, Orhan and Zoph, Barret and Fedus, Liam and Bosma, Maarten and Zhou, Zongwei and Wang, Tao and Wang, Yu Emma and Webster, Kellie and Pellat, Marie and Robinson, Kevin and Meier-Hellstern, Kathleen and Duke, Toju and Dixon, Lucas and Zhang, Kun and Le, Quoc V. and Wu, Yonghui and Chen, Zhifeng and Cui, Claire},
  journal = {arXiv preprint arXiv:2112.06905},
  year    = {2022}
}

@inproceedings{du2024stackingtransformers,
  title     = {Stacking Your Transformers: A Closer Look at Model Growth for Efficient {LLM} Pre-Training},
  author    = {Du, Wenyu and Luo, Tongxu and Qiu, Zihan and Huang, Zeyu and Shen, Yikang and Cheng, Reynold and Guo, Yike and Fu, Jie},
  booktitle = {Advances in Neural Information Processing Systems (NeurIPS)},
  year      = {2024}
}

@article{abnar2025optimalsparsity,
  title   = {Parameters vs {FLOPs}: Scaling Laws for Optimal Sparsity for Mixture-of-Experts Language Models},
  author  = {Abnar, Samira and Shah, Harshay and Busbridge, Dan and Ali, Alaaeldin Mohamed Elnouby and Susskind, Josh and Thilak, Vimal},
  journal = {arXiv preprint arXiv:2501.12370},
  year    = {2025}
}

@article{liew2025scalinglaws,
  title   = {Scaling Laws for Upcycling Mixture-of-Experts Language Models},
  author  = {Liew, Seng Pei and Kato, Takuya and Takase, Sho},
  journal = {arXiv preprint arXiv:2502.03009},
  year    = {2025}
}

@article{he2024upcyclingllm,
  title   = {Upcycling Large Language Models into Mixture of Experts},
  author  = {He, Ethan and Khattar, Abhinav and Prenger, Ryan and Korthikanti, Vijay and Yan, Zijie and Liu, Tong and Fan, Shiqing and Aithal, Ashwath and Shoeybi, Mohammad and Catanzaro, Bryan},
  journal = {arXiv preprint arXiv:2410.07524},
  year    = {2024}
}

@article{zhang2024bam,
  title   = {{BAM!} Just Like That: Simple and Efficient Parameter Upcycling for Mixture of Experts},
  author  = {Zhang, Qizhen and Gritsch, Nikolas and Gnaneshwar, Dwaraknath and Guo, Simon and Cairuz, David and Venkitesh, Bharat and Foerster, Jakob and Blunsom, Phil and Ruder, Sebastian and Ustun, Ahmet and Locatelli, Acyr},
  journal = {arXiv preprint arXiv:2408.08274},
  year    = {2024}
}

@article{muennighoff2023dataconstrained,
  title   = {Scaling Data-Constrained Language Models},
  author  = {Muennighoff, Niklas and Rush, Alexander M. and Barak, Boaz and Le Scao, Teven and Piktus, Aleksandra and Tazi, Nouamane and Pyysalo, Sampo and Wolf, Thomas and Raffel, Colin},
  journal = {arXiv preprint arXiv:2305.16264},
  year    = {2023}
}

@inproceedings{chi2022representationcollapse,
  title     = {On the Representation Collapse of Sparse Mixture of Experts},
  author    = {Chi, Zewen and Dong, Li and Huang, Shaohan and Dai, Damai and Ma, Shuming and Patra, Barun and Singhal, Saksham and Bajaj, Payal and Song, Xia and Mao, Xian-Ling and Huang, Heyan and Wei, Furu},
  booktitle = {Advances in Neural Information Processing Systems (NeurIPS)},
  year      = {2022}
}

@article{sukhbaatar2024btx,
  title   = {Branch-Train-{MiX}: Mixing Expert {LLMs} into a Mixture-of-Experts {LLM}},
  author  = {Sukhbaatar, Sainbayar and Golovneva, Olga and Sharma, Vasu and Xu, Hu and Lin, Xi Victoria and Rozi{\`e}re, Baptiste and Kahn, Jacob and Li, Daniel and Yih, Wen-tau and Weston, Jason and Li, Xian},
  journal = {arXiv preprint arXiv:2403.07816},
  year    = {2024}
}

@article{parmar2024reuse,
  title   = {Reuse, Don't Retrain: A Recipe for Continued Pretraining of Language Models},
  author  = {Parmar, Jupinder and Satheesh, Sanjev and Patwary, Mostofa and Shoeybi, Mohammad and Catanzaro, Bryan},
  journal = {arXiv preprint arXiv:2407.07263},
  year    = {2024}
}

@article{panigrahi2024raptr,
  title   = {Efficient Stagewise Pretraining via Progressive Subnetworks},
  author  = {Panigrahi, Abhishek and Saunshi, Nikunj and Lyu, Kaifeng and Miryoosefi, Sobhan and Reddi, Sashank and Kale, Satyen and Kumar, Sanjiv},
  journal = {arXiv preprint arXiv:2402.05913},
  year    = {2024}
}

@inproceedings{pan2023multilinear,
  title     = {Reusing Pretrained Models by Multi-linear Operators for Efficient Training},
  author    = {Pan, Yu and Yuan, Ye and Yin, Yichun and Xu, Zenglin and Shang, Lifeng and Jiang, Xin and Liu, Qun},
  booktitle = {Advances in Neural Information Processing Systems (NeurIPS)},
  year      = {2023}
}

@article{huang2025ders,
  title   = {{DeRS}: Towards Extremely Efficient Upcycled Mixture-of-Experts Models},
  author  = {Huang, Yongqi and Ye, Peng and Huang, Chenyu and Cao, Jianjian and Zhang, Lin and Li, Baopu and Yu, Gang and Chen, Tao},
  journal = {arXiv preprint arXiv:2503.01359},
  year    = {2025},
  note    = {Accepted at CVPR 2025}
}

@article{wu2025grovemoe,
  title   = {Grove {MoE}: Towards Efficient and Superior {MoE} {LLMs} with Adjugate Experts},
  author  = {Wu, Haoyuan and Chen, Haoxing and Chen, Xiaodong and Zhou, Zhanchao and Chen, Tieyuan and Zhuang, Yihong and Lu, Guoshan and Huang, Zenan and Zhao, Junbo and Liu, Lin and Lan, Zhenzhong and Yu, Bei and Li, Jianguo},
  journal = {arXiv preprint arXiv:2508.07785},
  year    = {2025}
}

@inproceedings{gritsch2025nexus,
  title     = {Nexus: Adaptive Upcycling to Efficiently Pretrain Mixture of Experts},
  author    = {Gritsch, Nikolas and Zhang, Qizhen and Locatelli, Acyr and Hooker, Sara and {\"U}st{\"u}n, Ahmet},
  booktitle = {Findings of the Association for Computational Linguistics: EMNLP 2025},
  pages     = {24364--24381},
  year      = {2025},
  publisher = {Association for Computational Linguistics}
}

@inproceedings{zhang2025layermoe,
  title     = {Less, but Better: Efficient Multilingual Expansion for {LLM}s via Layer-wise Mixture-of-Experts},
  author    = {Zhang, Xue and Chen, Yunlong and Liu, Tong and Wang, Cong and Liu, Mo and Huang, Hongji and Li, Yihan},
  booktitle = {Proceedings of the 63rd Annual Meeting of the Association for Computational Linguistics (ACL)},
  year      = {2025},
  note      = {arXiv:2505.22582}
}

@inproceedings{zhang2025bts,
  title     = {{BTS}: Harmonizing Specialized Experts into a Generalist {LLM}},
  author    = {Zhang, Qizhen and Bhargava, Prajjwal and Bi, Chloe and Cai, Chris X. and Foerster, Jakob Nicolaus and Fu, Jeremy and Koura, Punit Singh and Silva, Ruan and Shen, Sheng and Dinan, Emily and Gururangan, Suchin and Lewis, Mike},
  booktitle = {Proceedings of the 2025 Conference on Empirical Methods in Natural Language Processing (EMNLP)},
  pages     = {6816--6834},
  year      = {2025},
  note      = {arXiv:2502.00075}
}

@inproceedings{jiang2025improved,
  title     = {Improved Sparse Upcycling for Instruction Tuning},
  author    = {Jiang, Wangyi and Lu, Yaojie and Lin, Hongyu and Han, Xianpei and Sun, Le},
  booktitle = {Proceedings of the 31st International Conference on Computational Linguistics (COLING)},
  pages     = {9485--9498},
  year      = {2025}
}

@article{wang2026symphonymoe,
  title={Symphony-MoE: Harmonizing Disparate Pre-trained Models into a Coherent Mixture-of-Experts},
  author={Wang, Qi and Peng, Hanyang and Yu, Yue},
  journal={Proceedings of the AAAI Conference on Artificial Intelligence},
  year={2026},
  note={arXiv preprint arXiv:2509.18542},
}

@article{fan2024empiricalmoe,
  title={Towards an Empirical Understanding of MoE Design Choices},
  author={Fan, Dongyang and Messmer, Bettina and Jaggi, Martin},
  journal={arXiv preprint arXiv:2402.13089},
  year={2024},
}

@article{deepseekai2024deepseekv3,
  title={DeepSeek-V3 Technical Report},
  author={{DeepSeek-AI}},
  journal={arXiv preprint arXiv:2412.19437},
  year={2024},
}

@article{li2025slimmoe,
  title={SlimMoE: Structured Compression of Large MoE Models via Expert Slimming and Distillation},
  author={Li, Zichong and Liang, Chen and Zhang, Zixuan and Hong, Ilgee and Kim, Young Jin and Chen, Weizhu and Zhao, Tuo},
  journal={arXiv preprint arXiv:2506.18349},
  year={2025},
}

@inproceedings{zinkevich2003online,
  title={Online Convex Programming and Generalized Infinitesimal Gradient Ascent},
  author={Zinkevich, Martin},
  booktitle={Proceedings of the 20th International Conference on Machine Learning (ICML)},
  pages={928--936},
  year={2003}
}

@article{shalev2012online,
  title={Online Learning and Online Convex Optimization},
  author={Shalev-Shwartz, Shai},
  journal={Foundations and Trends in Machine Learning},
  volume={4},
  number={2},
  pages={107--194},
  year={2012}
}

@book{hazan2016introduction,
  title={Introduction to Online Convex Optimization},
  author={Hazan, Elad},
  year={2016},
  publisher={MIT Press},
  edition={2nd}
}

@inproceedings{schaipp2025surprising,
  title={The Surprising Agreement Between Convex Optimization Theory and Learning-Rate Scheduling for Large Model Training},
  author={Schaipp, Fabian and Defazio, Aaron and Mehta, Harsh and Mishchenko, Konstantin and Khaled, Ahmed},
  booktitle={Proceedings of the 42nd International Conference on Machine Learning (ICML)},
  year={2025}
}

@article{bu2026convex,
  title   = {Convex Dominance in Deep Learning {I}: A Scaling Law of Loss and Learning Rate},
  author  = {Bu, Zhiqi and Xu, Shiyun and Mao, Jialin},
  journal = {arXiv preprint arXiv:2602.07145},
  year    = {2026},
  note    = {Accepted to ICLR 2026}
}

@misc{metalllama4,
  author = {{Meta AI}},
  title = {Llama 4: Natively Multimodal Foundation Models},
  year = {2025},
  howpublished = {\url{https://github.com/meta-llama/llama-models}},
  note = {Model card available at \url{https://github.com/meta-llama/llama-models/blob/main/models/llama4/MODEL_CARD.md}}
}

@article{du2024revisitingmoe,
  title   = {Revisiting {MoE} and Dense Speed-Accuracy Comparisons for {LLM} Training},
  author  = {Du, Xianzhi and Gunter, Tom and Kong, Xiang and Lee, Mark and Wang, Zirui and Zhang, Aonan and Du, Nan and Pang, Ruoming},
  journal = {arXiv preprint arXiv:2405.15052},
  year    = {2024}
}

@article{jin2025megascalemoe,
  title   = {{MegaScale-MoE}: Large-Scale Communication-Efficient Training of Mixture-of-Experts Models in Production},
  author  = {Jin, Chao and Jiang, Ziheng and Bai, Zhihao and Zhong, Zheng and Liu, Juncai and Li, Xiang and Zheng, Ningxin and Wang, Xi and Xie, Cong and Huang, Qi and Heng, Wen and Ma, Yiyuan and Bao, Wenlei and Zheng, Size and Peng, Yanghua and Lin, Haibin and Liu, Xuanzhe and Jin, Xin and Liu, Xin},
  journal = {arXiv preprint arXiv:2505.11432},
  year    = {2025}
}

@article{kimik2,
  title={Kimi K2: Open Agentic Intelligence},
  author={Moonshot-AI},
  year={2025},
  journal={Technical Report},
  url={https://moonshotai.github.io/Kimi-K2/}
}

@article{qwen3,
  title={Qwen3 Technical Report},
  author={Qwen Team},
  year={2025},
  journal={arXiv preprint arXiv:2505.09388}
}

@article{zadouri2023moe,
  title={Pushing Mixture of Experts to the Limit: Extremely Parameter Efficient MoE for Instruction Tuning},
  author={Zadouri, Ted and {\"U}st{\"u}n, Ahmet and Ahmadian, Arash and Ermi{\c{s}}, Beyza and Locatelli, Acyr and Hooker, Sara},
  journal={arXiv preprint arXiv:2309.05444},
  year={2023}
}

@inproceedings{nie2024lshmoe,
  title={{LSH-MoE}: Communication-efficient {MoE} Training via Locality-Sensitive Hashing},
  author={Nie, Xiaonan and Liu, Qibin and Fu, Fangcheng and Zhu, Shenhan and Miao, Xupeng and Li, Xiaoyang and Zhang, Yang and Liu, Shouda and Cui, Bin},
  booktitle={Advances in Neural Information Processing Systems},
  volume={37},
  year={2024}
}

@article{yu2026sparkling,
  title={{SPARKLING}: Balancing Signal Preservation and Symmetry Breaking for Width-Progressive Learning},
  author={Yu, Qifan and Ma, Xinyu and Zhuo, Zhijian and Wang, Minrui and Liu, Deyi and Zhan, Shiyi and Ma, Yiyuan and Xiang, Liang and Bin, Xingyan and He, Di},
  journal={arXiv preprint arXiv:2602.02472},
  year={2026}
}

@article{wei2022emergent,
  title={Emergent abilities of large language models},
  author={Wei, Jason and Tay, Yi and Bommasani, Rishi and Raffel, Colin and Zoph, Barret and Borgeaud, Sebastian and Yogatama, Dani and Bosma, Maarten and Zhou, Denny and Metzler, Donald and Chi, Ed H. and Hashimoto, Tatsunori and Vinyals, Oriol and Liang, Percy and Dean, Jeff and Fedus, William},
  journal={Transactions on Machine Learning Research},
  year={2022}
}
}

\appendix

\clearpage

\section*{Appendix}

\vspace{0.5em}
\noindent\textbf{Contents}
\vspace{0.3em}

{
\hypersetup{hidelinks}
\renewcommand{\arraystretch}{1.15}
\noindent
\begin{tabular}{@{}p{0.82\linewidth}@{}r@{}}
Regret-bound decomposition for expert upcycling & \ref{app:proof} \\
Theoretical Justification for Gradient-Based Utility Scores & \ref{app:utility_theory} \\
Model Configurations & \ref{app:model_configs} \\
Heuristic Upcycling Methods and Results & \ref{app:heuristic_upcycling_methods} \\
Extended Related Work & \ref{app:extended_related_work} \\
\end{tabular}
}

\vspace{1em}

\section{Regret-bound decomposition for expert upcycling}
\label{app:proof}

Using the online convex optimization (OCO) framework~\citep{zinkevich2003online, shalev2012online, hazan2016introduction}, we adapt the regret-telescoping approach of~\citet{bu2025dpt} to formalize the two-term decomposition introduced in \S\ref{subsec:operator}.

\subsection{Setup}

\paragraph{Notation.}
Let $z \sim \mathcal{D}$ denote tokens from the pretraining distribution and $\ell(\cdot)$ the token-level cross-entropy loss.  For an MoE Transformer with top-$K$ routing, define the population objective for a model with $n$ experts as $\mathcal{L}_n(\theta) := \mathbb{E}_{z \sim \mathcal{D}}[\ell(f_n(z;\theta))]$, where $\theta$ collects all parameters and $\Theta_n$ denotes the parameter space.  Write $\mathcal{L}_E$ for the $E$-expert objective, $\mathcal{L}_{mE}$ for the expanded $mE$-expert objective, and $\mathcal{L}_n^\star = \min_{\theta \in \Theta_n} \mathcal{L}_n(\theta)$.  Partition the $mE$-expert parameter vector as $\theta = (\theta_s, \theta_+)$, where $\theta_s$ denotes parameters shared with the $E$-expert model (dense layers, embeddings, original experts) and $\theta_+$ denotes the degrees of freedom introduced by expansion (replica experts, expanded router weights).  Let $\theta_{mE}^\star = (\theta_s^\star, \theta_+^\star) \in \arg\min \mathcal{L}_{mE}$, and let $\theta_+^{U}$ and $\theta_+^{0}$ denote the new-parameter initialization at step~$\tau$ produced by the upcycling operator (Definition~\ref{def:upcycling_operator}) and the random initialization used by the fixed-size procedure, respectively.  We compare two procedures over $T$ total gradient steps with learning-rate schedule $\{\eta_t\}_{t=0}^{T-1}$: (i)~\emph{expert upcycling}: train an $E$-expert model for $\tau$ steps, expand to $mE$ experts, continue for $T - \tau$ steps; (ii)~\emph{fixed-size}: train an $mE$-expert model for all $T$ steps from random initialization.

\paragraph{Assumptions.}
\begin{assumption}[Convexity]
\label{asm:convex}
$\mathcal{L}_{mE}(\theta)$ is convex in $\theta$.
\end{assumption}
\begin{assumption}[Bounded gradients]
\label{asm:bounded_grad}
$\|\nabla\mathcal{L}_n(\theta)\|_2 \le G$ for all $\theta$ and $n \in \{E, mE\}$.
\end{assumption}
These are standard in the OCO literature \citep{zinkevich2003online, hazan2016introduction}.  They do not hold literally for deep networks; we adopt them to derive structural insights rather than tight numerical bounds.  Recent work shows that convex optimization theory is surprisingly predictive of large-scale training dynamics despite non-convexity \citep{schaipp2025surprising, bu2026convex}.

\paragraph{Canonical lifting.}
\begin{definition}[Canonical lifting]
\label{def:canonical_lifting}
The \emph{canonical lifting} $\iota: \Theta_E \to \Theta_{mE}$ retains the original $E$ experts and router unchanged, sets the extra $(m{-}1)E$ expert weights to zero, and sets their router logits to a sufficiently large negative constant $-M$ (chosen so the extra experts are never selected by top-$K$; because top-$K$ is a discrete hard threshold, any finite $M$ strictly larger than the maximum attained logit gap suffices).  Since the extra experts are never selected by top-$K$:
\begin{enumerate}[nosep,label=(\alph*)]
  \item $\mathcal{L}_{mE}(\iota(\theta_E)) = \mathcal{L}_E(\theta_E)$ for all $\theta_E$ \hfill (loss preservation),
  \item $\nabla_{\theta_+}\mathcal{L}_{mE}(\iota(\theta_E)) = \mathbf{0}$ \hfill (zero gradient on new parameters).
\end{enumerate}
Property~(b) holds because zero-weight experts with zero gating contribute nothing to the forward pass.  The lifting $\iota$ is a proof device: it lets us represent Phase~1 iterates in $\Theta_{mE}$ for the telescoping argument without changing the optimization dynamics.
\end{definition}

\subsection{Theorem and proof}

\begin{theorem}[Expert upcycling bound]
\label{thm:progressive_bound}
Let $\bar{L}^{\mathrm{up}}$ and $\bar{L}^{\mathrm{fs}}$ denote the learning-rate-weighted average training losses of the two procedures over the shared schedule $\{\eta_t\}_{t=0}^{T-1}$, and let $\mathcal{R}^{\mathrm{up}}$ and $\mathcal{R}^{\mathrm{fs}}$ denote their respective OCO regret upper bounds (defined in Step~3 of the derivation below):
\begin{equation*}
  \bar{L}^{\mathrm{up}} \;\le\; \mathcal{R}^{\mathrm{up}},
  \qquad
  \bar{L}^{\mathrm{fs}} \;\le\; \mathcal{R}^{\mathrm{fs}}.
\end{equation*}
Suppose both procedures share the same initial $\theta_s$ and that the shared components of their respective optima coincide.  Then under Assumptions~\ref{asm:convex}--\ref{asm:bounded_grad}, these regret bounds differ by:
\begin{equation}
  \mathcal{R}^{\mathrm{up}} - \mathcal{R}^{\mathrm{fs}}
  \;=\;
  \underbrace{
    \frac{\sum_{t=0}^{\tau-1}\eta_t}
         {\sum_{t=0}^{T-1}\eta_t}
    \bigl(\mathcal{L}_E^\star
          - \mathcal{L}_{mE}^\star\bigr)
  }_{\text{(I)\; capacity gap}}
  \;+\;
  \underbrace{
    \frac{\|\theta_+^{U} - \theta_+^\star\|^2
          - \|\theta_+^{0} - \theta_+^\star\|^2}
         {2\sum_{t=0}^{T-1}\eta_t}
  }_{\text{(II)\; initialization gain}}.
  \label{eq:main_bound}
\end{equation}
\end{theorem}

The proof proceeds in four steps.

\paragraph{Step 1: one-step descent lemma.}
\begin{lemma}[Gradient-descent regret inequality]
\label{lem:one_step}
Let $\mathcal{L}$ satisfy Assumptions~\ref{asm:convex}--\ref{asm:bounded_grad}.  For any comparator $\theta^\star$ and iterate $\theta_{t+1} = \theta_t - \eta_t\nabla\mathcal{L}(\theta_t)$:
\begin{equation}
  \eta_t\bigl(\mathcal{L}(\theta_t) - \mathcal{L}(\theta^\star)\bigr)
  \;\le\;
  \tfrac{1}{2}\bigl(\|\theta_t - \theta^\star\|^2 - \|\theta_{t+1} - \theta^\star\|^2\bigr)
  + \tfrac{1}{2}\eta_t^2 G^2.
  \label{eq:one_step}
\end{equation}
\end{lemma}
\begin{proof}
Expand the squared distance after the update:
\begin{align}
  \|\theta_{t+1} - \theta^\star\|^2
  &= \|\theta_t - \theta^\star\|^2
     - 2\eta_t\langle \nabla\mathcal{L}(\theta_t),\, \theta_t - \theta^\star\rangle
     + \eta_t^2\|\nabla\mathcal{L}(\theta_t)\|^2.
  \label{eq:expand}
\end{align}
By convexity (Assumption~\ref{asm:convex}): $\langle \nabla\mathcal{L}(\theta_t),\, \theta_t - \theta^\star\rangle \ge \mathcal{L}(\theta_t) - \mathcal{L}(\theta^\star)$.  By the gradient bound (Assumption~\ref{asm:bounded_grad}): $\|\nabla\mathcal{L}(\theta_t)\|^2 \le G^2$.  Substituting into \eqref{eq:expand} and rearranging yields \eqref{eq:one_step}.  This is the standard OCO regret inequality for online gradient descent \citep[Theorem~1]{zinkevich2003online}.
\end{proof}

\paragraph{Step 2: phase-wise telescoping.}
\emph{Phase 1 ($t = 0, \ldots, \tau-1$): training in $\Theta_E$.}  We represent the upcycling iterates in $\Theta_{mE}$ via the lifting $\iota$ (Definition~\ref{def:canonical_lifting}).  Write $\tilde\theta_t = \iota(\theta_t^{\mathrm{up}})$ for $t \le \tau$.  By property~(a) of $\iota$, $\mathcal{L}_{mE}(\tilde\theta_t) = \mathcal{L}_E(\theta_t^{\mathrm{up}})$.  By property~(b), the new-parameter coordinates remain at zero throughout Phase~1, so the lifted iterates follow the same trajectory as the $E$-expert SGD.  Choose comparator $\iota(\theta_E^\star)$, where $\theta_E^\star \in \arg\min \mathcal{L}_E$.  Applying Lemma~\ref{lem:one_step} at each step $t = 0, \ldots, \tau-1$ and summing (intermediate distance terms telescope, leaving only boundary terms):
\begin{equation}
  \sum_{t=0}^{\tau-1}\eta_t\bigl(\mathcal{L}_E(\theta_t^{\mathrm{up}}) - \mathcal{L}_E^\star\bigr)
  \;\le\;
  \tfrac{1}{2}\|\tilde\theta_0 - \iota(\theta_E^\star)\|^2
  - \tfrac{1}{2}\|\tilde\theta_\tau - \iota(\theta_E^\star)\|^2
  + \tfrac{G^2}{2}\sum_{t=0}^{\tau-1}\eta_t^2.
  \label{eq:phase1}
\end{equation}
\emph{Phase 2 ($t = \tau, \ldots, T-1$): training in $\Theta_{mE}$.}  At step $\tau$, the upcycling operator is applied to the Phase-1 terminal iterate to obtain $\theta_\tau^{\mathrm{up}} \in \Theta_{mE}$.  From this point, iterates live in $\Theta_{mE}$ directly.  Choose comparator $\theta_{mE}^\star \in \arg\min \mathcal{L}_{mE}$.  Applying Lemma~\ref{lem:one_step} at each step $t = \tau, \ldots, T-1$ and summing:
\begin{equation}
  \sum_{t=\tau}^{T-1}\eta_t\bigl(\mathcal{L}_{mE}(\theta_t^{\mathrm{up}}) - \mathcal{L}_{mE}^\star\bigr)
  \;\le\;
  \tfrac{1}{2}\|\theta_\tau^{\mathrm{up}} - \theta_{mE}^\star\|^2
  - \tfrac{1}{2}\|\theta_T^{\mathrm{up}} - \theta_{mE}^\star\|^2
  + \tfrac{G^2}{2}\sum_{t=\tau}^{T-1}\eta_t^2.
  \label{eq:phase2}
\end{equation}

\paragraph{Step 3: combining phases.}
Adding \eqref{eq:phase1} and \eqref{eq:phase2}, and writing $L_t^{\mathrm{up}} = \mathcal{L}_E(\theta_t^{\mathrm{up}})$ for $t < \tau$ and $L_t^{\mathrm{up}} = \mathcal{L}_{mE}(\theta_t^{\mathrm{up}})$ for $t \ge \tau$, we obtain a bound on the $\eta$-weighted sum of losses.  Of the two terminal distance terms on the right side, the Phase-$T$ term $-\tfrac{1}{2}\|\theta_T^{\mathrm{up}} - \theta_{mE}^\star\|^2 \le 0$ is dropped (yielding a looser upper bound); the Phase-$1$ terminal $-\tfrac{1}{2}\|\tilde\theta_\tau - \iota(\theta_E^\star)\|^2$ is retained, since its shared-parameter component will cancel a matching term at the transition in Step~4.  This gives
\begin{multline}
  \sum_{t=0}^{T-1}\eta_t\, L_t^{\mathrm{up}}
  \;\le\;
  \Bigl(\sum_{t=0}^{\tau-1}\eta_t\Bigr)\mathcal{L}_E^\star + \Bigl(\sum_{t=\tau}^{T-1}\eta_t\Bigr)\mathcal{L}_{mE}^\star
  \;+\;
  \tfrac{G^2}{2}\sum_{t=0}^{T-1}\eta_t^2 \\
  \;+\;
  \tfrac{1}{2}\bigl(\|\tilde\theta_0 - \iota(\theta_E^\star)\|^2
      - \|\tilde\theta_\tau - \iota(\theta_E^\star)\|^2
      + \|\theta_\tau^{\mathrm{up}} - \theta_{mE}^\star\|^2\bigr).
  \label{eq:combined_before_div}
\end{multline}
Dividing by $\sum_{t=0}^{T-1}\eta_t$ gives $\bar{L}^{\mathrm{up}} \le \mathcal{R}^{\mathrm{up}}$, where $\bar{L}^{\mathrm{up}} := (\sum_t \eta_t L_t^{\mathrm{up}}) / \sum_t \eta_t$ is the learning-rate-weighted average loss and the \emph{upcycling regret upper bound} is defined as:
\begin{multline}
  \mathcal{R}^{\mathrm{up}}
  \;:=\;
  \frac{\bigl(\sum_{t=0}^{\tau-1}\eta_t\bigr)\mathcal{L}_E^\star + \bigl(\sum_{t=\tau}^{T-1}\eta_t\bigr)\mathcal{L}_{mE}^\star}{\sum_{t=0}^{T-1}\eta_t}
  \;+\;
  \frac{G^2\sum_{t=0}^{T-1}\eta_t^2}{2\sum_{t=0}^{T-1}\eta_t} \\
  \;+\;
  \frac{\|\tilde\theta_0 - \iota(\theta_E^\star)\|^2 - \|\tilde\theta_\tau - \iota(\theta_E^\star)\|^2 + \|\theta_\tau^{\mathrm{up}} - \theta_{mE}^\star\|^2}{2\sum_{t=0}^{T-1}\eta_t}.
  \label{eq:combined_up}
\end{multline}
For the fixed-size procedure, which trains in $\Theta_{mE}$ for all $T$ steps with comparator $\theta_{mE}^\star$, applying Lemma~\ref{lem:one_step}, dropping the non-positive terminal $-\tfrac{1}{2}\|\theta_T^{\mathrm{fs}} - \theta_{mE}^\star\|^2$, and dividing by $\sum_{t=0}^{T-1}\eta_t$ analogously gives $\bar{L}^{\mathrm{fs}} \le \mathcal{R}^{\mathrm{fs}}$, where the \emph{fixed-size regret upper bound} is defined as:
\begin{equation}
  \mathcal{R}^{\mathrm{fs}}
  \;:=\;
  \mathcal{L}_{mE}^\star
  \;+\;
  \frac{\|\theta_0^{\mathrm{fs}} - \theta_{mE}^\star\|^2}{2\sum_{t=0}^{T-1}\eta_t}
  \;+\;
  \frac{G^2\sum_{t=0}^{T-1}\eta_t^2}{2\sum_{t=0}^{T-1}\eta_t}.
  \label{eq:combined_fs}
\end{equation}

\paragraph{Step 4: difference of the regret upper bounds.}
We now compute $\mathcal{R}^{\mathrm{up}} - \mathcal{R}^{\mathrm{fs}}$ as an algebraic identity between the two explicit expressions \eqref{eq:combined_up} and \eqref{eq:combined_fs}.  This is a direct computation on the bounds themselves, not a bound on the loss gap $\bar{L}^{\mathrm{up}} - \bar{L}^{\mathrm{fs}}$.  Subtracting \eqref{eq:combined_fs} from \eqref{eq:combined_up} gives three groups of terms: a comparator-loss difference, a distance-term difference, and a $G^2$ contribution,
\begin{align}
  \mathcal{R}^{\mathrm{up}} - \mathcal{R}^{\mathrm{fs}}
  \;=\;&\;
  \underbrace{\frac{\bigl(\sum_{t=0}^{\tau-1}\eta_t\bigr)\mathcal{L}_E^\star + \bigl(\sum_{t=\tau}^{T-1}\eta_t\bigr)\mathcal{L}_{mE}^\star}{\sum_{t=0}^{T-1}\eta_t} - \mathcal{L}_{mE}^\star}_{\text{comparator-loss difference}} \notag \\
  &\;+\;
  \underbrace{\frac{\|\tilde\theta_0 - \iota(\theta_E^\star)\|^2 - \|\tilde\theta_\tau - \iota(\theta_E^\star)\|^2 + \|\theta_\tau^{\mathrm{up}} - \theta_{mE}^\star\|^2 - \|\theta_0^{\mathrm{fs}} - \theta_{mE}^\star\|^2}{2\sum_{t=0}^{T-1}\eta_t}}_{\text{distance-term difference}} \notag \\
  &\;+\;
  \underbrace{\frac{G^2\sum_{t=0}^{T-1}\eta_t^2}{2\sum_{t=0}^{T-1}\eta_t} - \frac{G^2\sum_{t=0}^{T-1}\eta_t^2}{2\sum_{t=0}^{T-1}\eta_t}}_{ = \; 0}.
  \label{eq:subtracted}
\end{align}
The $G^2$ terms are identical in both bounds (same schedule, same $G$) and cancel exactly.  For the comparator-loss difference, writing $\sum_{t=\tau}^{T-1}\eta_t = \sum_{t=0}^{T-1}\eta_t - \sum_{t=0}^{\tau-1}\eta_t$ and simplifying:
\begin{equation}
  \frac{\bigl(\sum_{t=0}^{\tau-1}\eta_t\bigr)\mathcal{L}_E^\star + \bigl(\sum_{t=\tau}^{T-1}\eta_t\bigr)\mathcal{L}_{mE}^\star}{\sum_{t=0}^{T-1}\eta_t} - \mathcal{L}_{mE}^\star
  \;=\;
  \frac{\sum_{t=0}^{\tau-1}\eta_t}{\sum_{t=0}^{T-1}\eta_t}\bigl(\mathcal{L}_E^\star - \mathcal{L}_{mE}^\star\bigr).
  \label{eq:comparator_simplified}
\end{equation}
For the distance-term difference, decompose $\theta = (\theta_s, \theta_+)$ so that $\|\theta - \theta^\star\|^2 = \|\theta_s - \theta_s^\star\|^2 + \|\theta_+ - \theta_+^\star\|^2$.  Both procedures use the same initial shared parameters ($\tilde\theta_{0,s} = \theta_{0,s}^{\mathrm{fs}}$), and assuming the shared components of $\theta_E^\star$ and $\theta_{mE}^\star$ coincide (the same simplification used in progressive depth expansion \citep{bu2025dpt}), $\iota(\theta_E^\star)_s = \theta_{mE,s}^\star$.  Under the lifting, both $\tilde\theta_0$ and $\iota(\theta_E^\star)$ have zero new-parameter coordinates, so
\begin{align*}
  \|\tilde\theta_0 - \iota(\theta_E^\star)\|^2
    &= \|\theta_{0,s}^{\mathrm{fs}} - \theta_{mE,s}^\star\|^2,
  &
  \|\theta_0^{\mathrm{fs}} - \theta_{mE}^\star\|^2
    &= \|\theta_{0,s}^{\mathrm{fs}} - \theta_{mE,s}^\star\|^2 + \|\theta_+^0 - \theta_+^\star\|^2,
\end{align*}
whose difference leaves $-\|\theta_+^0 - \theta_+^\star\|^2$.  At the transition, the upcycling iterate satisfies $\|\theta_\tau^{\mathrm{up}} - \theta_{mE}^\star\|^2 = \|\theta_{\tau,s}^{\mathrm{up}} - \theta_{mE,s}^\star\|^2 + \|\theta_+^U - \theta_+^\star\|^2$, and under the lifting $\|\tilde\theta_\tau - \iota(\theta_E^\star)\|^2 = \|\theta_{\tau,s}^{\mathrm{up}} - \theta_{mE,s}^\star\|^2$ (new coordinates still zero).  The retained Phase-1 terminal $-\|\tilde\theta_\tau - \iota(\theta_E^\star)\|^2$ exactly cancels the shared-parameter piece of $+\|\theta_\tau^{\mathrm{up}} - \theta_{mE}^\star\|^2$, leaving $+\|\theta_+^U - \theta_+^\star\|^2$.  Thus the distance-term difference reduces to
\begin{equation}
  \frac{\|\theta_+^U - \theta_+^\star\|^2 - \|\theta_+^0 - \theta_+^\star\|^2}{2\sum_{t=0}^{T-1}\eta_t}.
  \label{eq:distance_simplified}
\end{equation}
Substituting \eqref{eq:comparator_simplified} and \eqref{eq:distance_simplified} into \eqref{eq:subtracted} yields $\mathcal{R}^{\mathrm{up}} - \mathcal{R}^{\mathrm{fs}} = (\mathrm{I}) + (\mathrm{II})$, which is Eq.~\eqref{eq:main_bound} and establishes Theorem~\ref{thm:progressive_bound}.  \qed

\subsection{Discussion of assumptions}
\label{app:discussion_assumptions}

\paragraph{Convexity.}
The convex assumption does not hold for deep networks.  However, a growing body of evidence shows that convex optimization theory provides qualitatively accurate predictions for large-scale training: \citet{schaipp2025surprising} demonstrate that convex SGD bounds closely predict optimal learning-rate schedules for LLM pre-training; \citet{bu2026convex} show that optimization dynamics are ``empirically convex-like'' across diverse tasks and models.  We adopt the convex framework for structural insight, not tight numerical bounds.

\paragraph{Parameter decomposition.}
The assumption that the shared components of $\theta_E^\star$ and $\theta_{mE}^\star$ coincide is a simplification: adding experts may shift the optimal dense-layer parameters.  This is the same assumption used in progressive depth expansion \citep{bu2025dpt} and enables a clean decomposition.  For expert-count expansion with $m=2$, a natural choice is to designate the first copy of each expert as ``shared'' and the second as ``new.''

\clearpage
\section{Theoretical justification for gradient-based utility scores}
\label{app:utility_theory}

We derive the two utility scores used in Section~\ref{sec:utility_upcycling} from first principles, starting from a local approximation of the loss at transition time $\tau$.

\paragraph{Setup.}
Let $\theta_E^\tau \in \Theta_E$ denote the trained $E$-expert checkpoint at transition time $\tau$. Partition the parameters as $\theta = (w_1, \ldots, w_E, \phi)$, where $w_e \in \mathbb{R}^{d_e}$ are the parameters of expert $e$ and $\phi$ collects all remaining parameters (dense layers, embeddings, router). Write $\mathcal{L}(\theta) = \mathbb{E}_{z \sim \mathcal{D}}[\ell(f(z;\theta))]$ for the population loss. Let $g_e = \nabla_{w_e} \mathcal{L}(\theta_E^\tau)$ denote the gradient with respect to expert $e$'s parameters, evaluated at transition time $\tau$.

\paragraph{First-order loss expansion.}
Consider perturbing expert $e$'s parameters by $\Delta w_e$ while holding all other parameters fixed. A first-order Taylor expansion around $\theta_E^\tau$ gives:
\begin{equation}
\mathcal{L}(\theta_E^\tau + \Delta w_e \mathbf{1}_e) \;=\; \mathcal{L}(\theta_E^\tau) \;+\; g_e^\top \Delta w_e \;+\; O(\|\Delta w_e\|_2^2),
\label{eq:taylor_first}
\end{equation}
where $\mathbf{1}_e$ denotes the indicator that only expert $e$'s block is perturbed.

\paragraph{Worst-case sensitivity and Utility 1: $u_G(e) = \|g_e\|_2^2$.}
By the Cauchy--Schwarz inequality, $|g_e^\top \Delta w_e| \le \|g_e\|_2 \cdot \|\Delta w_e\|_2$, with equality when $\Delta w_e \propto g_e$. For a unit perturbation, the worst-case first-order loss change is $\|g_e\|_2$. Ranking experts by $\|g_e\|_2$ thus ranks them by how much the loss can change per unit perturbation. We use the square for convenience:
\begin{equation}
u_G(e) \;:=\; \|g_e\|_2^2.
\label{eq:utility_g2}
\end{equation}
An expert with large $u_G(e)$ is one where the loss landscape is steep: the objective is actively sensitive to this expert's parameters under the current data distribution and routing at time $\tau$. Replicating such an expert introduces new degrees of freedom precisely where the loss is most responsive, giving CPT the greatest opportunity to reduce the initialization gap term in Eq.~\eqref{eq:main_bound}.

\paragraph{Scale-aware sensitivity and Utility 2: $\|w_e\|_2 \cdot \|g_e\|_2$.}
The gradient norm $\|g_e\|_2$ is not scale-invariant: if expert $e$'s parameters are uniformly rescaled by $\alpha > 0$, the gradient scales as $g_e \to g_e / \alpha$, so $\|g_e\|_2$ decreases even though the expert's functional contribution is unchanged. This can cause $u_G$ to systematically underrank large-norm experts that are functionally important but whose gradients have been reduced by scale.

To correct for this, consider perturbing expert $e$'s parameters proportionally to their current magnitude, $\Delta w_e = \epsilon \cdot w_e$:
\begin{equation*}
\mathcal{L}(\theta_E^\tau + \epsilon w_e \mathbf{1}_e) - \mathcal{L}(\theta_E^\tau) \;\approx\; \epsilon \cdot g_e^\top w_e.
\end{equation*}
The absolute value is bounded by $|g_e^\top w_e| \le \|g_e\|_2 \cdot \|w_e\|_2$, tight when $g_e \propto w_e$. The product $\|w_e\|_2 \cdot \|g_e\|_2$ captures worst-case loss sensitivity under proportional perturbations. We define:
\begin{equation}
u_{\mathrm{SAL}}(e) \;:=\; \|w_e\|_2 \cdot \|g_e\|_2.
\label{eq:utility_sal}
\end{equation}
This is the weight-space analogue of the Taylor saliency criterion of \citet{molchanov2017taylor}, who derive $\Theta_{TE}(h_i) = |\frac{\partial \mathcal{C}}{\partial h_i} h_i|$ in activation space. The weight-space version $|g_{ij} w_{ij}|$ has been used as a structured pruning criterion in recent work \citep{li2025slimmoe}. Our criterion operates at the expert block level rather than individual scalars.

In practice, $u_G$ and $u_{\mathrm{SAL}}$ perform similarly to each other and both significantly outperform uniform copy-paste, suggesting that any gradient-based importance signal is more informative than treating all experts as equally valuable replication targets.

\paragraph{Why not second-order?}
The second-order term in Eq.~\eqref{eq:taylor_first} involves the Hessian block $H_e = \nabla^2_{w_e} \mathcal{L}(\theta_E^\tau)$. Curvature-normalized scores such as $g_e^\top H_e^{-1} g_e$ \citep{hassibi1992second} are theoretically more precise but require estimating $H_e$, which is expensive and noisy in practice. Diagonal Fisher approximations introduce significant bias \citep{soen2024tradeoffs}, and in our experiments curvature-normalized variants did not outperform the first-order scores. We therefore use $u_G$ and $u_{\mathrm{SAL}}$ as our primary utilities.

\section{Model configurations}
\label{app:model_configs}

This appendix provides the complete model configurations used across all experiments (see \S~\ref{sec:setup} for the experimental setup). Table~\ref{tab:model_configs_interleaved} details the interleaved MoE architecture configuration, with corresponding parameter counts and compute statistics in Tables~\ref{tab:arch_config_interleaved}--\ref{tab:param_counts_interleaved}. Tables~\ref{tab:model_configs_fullmoe}--\ref{tab:training_config} provide the full (non-interleaved) MoE architecture (Table~\ref{tab:arch_config_fullmoe}), parameter counts (Table~\ref{tab:param_counts_fullmoe}), and training configurations.

\begin{table*}[htbp]
\centering
\caption{Interleaved MoE model configurations and statistics. All models share: Context Length = 8192, Vocab Size = 200704, Grouped Query Attention (GQA), loss-free load balancing, 32 routed experts, and 0 shared experts.}
\label{tab:model_configs_interleaved}
\begin{subtable}{\textwidth}
\centering
\caption{Architecture configuration. Layers are interleaved dense and MoE; Act.\ FFN $= \text{top-$K$} \times \text{FFN/Expert}$; Total FFN $= 32 \times \text{FFN/Expert}$.}
\label{tab:arch_config_interleaved}
\resizebox{\textwidth}{!}{%
\begin{tabular}{@{}lcccccccccc@{}}
\toprule
\textbf{Model} & \textbf{Layers} & \textbf{Hidden} & \textbf{FFN} & \textbf{Attn} & \textbf{Attn} & \textbf{MoE} & \textbf{FFN/} & \textbf{top-$K$} & \textbf{Act. FFN} & \textbf{Total FFN} \\
 & & \textbf{Dim} & \textbf{Dim} & \textbf{Heads} & \textbf{Groups} & \textbf{Layers} & \textbf{Expert} & & \textbf{(MoE)} & \textbf{(MoE)} \\
\midrule
\multicolumn{11}{c}{\textit{20-layer}} \\
20-layer (baseline) & 20 & 2048 & 7168 & 16 & 16 & 10 & 3072 & 2 & 6144 & 98304 \\
\midrule
\multicolumn{11}{c}{\textit{10-layer}} \\
10-layer (top2) & 10 & 1024 & 3584 & 8 & 8 & 5 & 1536 & 2 & 3072 & 49152 \\
\midrule
\multicolumn{11}{c}{\textit{8-layer}} \\
8-layer (top1) & 8 & 768 & 2688 & 6 & 6 & 4 & 1152 & 1 & 1152 & 36864 \\
\bottomrule
\end{tabular}%
}
\end{subtable}

\vspace{1em}

\begin{subtable}{\textwidth}
\centering
\caption{Parameter counts (in millions) and compute. Train Tokens (B) are the actual pre-training token budgets used in experiments.}
\label{tab:param_counts_interleaved}
\resizebox{\textwidth}{!}{%
\begin{tabular}{@{}lS[table-format=2.1]S[table-format=3.1]S[table-format=3.1]S[table-format=3.1]S[table-format=4.1]S[table-format=4.1]S[table-format=4.1]S[table-format=3.1]S[table-format=1.2e1]@{}}
\toprule
\textbf{Model} & {\textbf{Router}} & {\textbf{Dense}} & {\textbf{MoE Act.}} & {\textbf{MoE Total}} & {\textbf{Embed}} & {\textbf{Act. Params}} & {\textbf{Total Params}} & {\textbf{Train}} & {\textbf{Total Train}} \\
 & {\textbf{(M)}} & {\textbf{Layer (M)}} & {\textbf{Layer (M)}} & {\textbf{Layer (M)}} & {\textbf{(M)}} & {\textbf{w/ Emb (M)}} & {\textbf{w/ Emb (M)}} & {\textbf{Tokens (B)}} & {\textbf{FLOPs}} \\
\midrule
\multicolumn{10}{c}{\textit{20-layer}} \\
20-layer (baseline) & 0.07 & 60.8 & 54.6 & 620.8 & 822.1 & 1976.3 & 7638.6 & 383.5 & 3.43e21 \\
\midrule
\multicolumn{10}{c}{\textit{10-layer}} \\
10-layer (top2) & 0.03 & 15.2 & 13.7 & 155.2 & 411.0 & 555.4 & 1263.2 & 92.6 & 1.27e20 \\
\midrule
\multicolumn{10}{c}{\textit{8-layer}} \\
8-layer (top1) & 0.02 & 8.6 & 5.0 & 87.3 & 308.3 & 362.6 & 691.8 & 59.5 & 3.36e19 \\
\bottomrule
\end{tabular}%
}
\end{subtable}

\end{table*}

\begin{table*}[htbp]
\centering
\caption{Full (non-interleaved) MoE model configurations and statistics. All models share: Context Length = 8192, Vocab Size = 200704, Grouped Query Attention (GQA), loss-free load balancing, 256 routed experts with Top-8 routing, and 0 shared experts.}
\label{tab:model_configs_fullmoe}
\vspace{0.5em}

\begin{subtable}{\textwidth}
\centering
\caption{Architecture configuration. Each configuration has two dense layers (prefix and suffix); the remaining layers are MoE. Exp FFN = FFN dimension per expert.}
\label{tab:arch_config_fullmoe}
\begin{adjustbox}{max width=\textwidth}
\begin{tabular}{@{}cccccccccc@{}}
\toprule
\multicolumn{5}{c}{\textbf{General Config}} & \textbf{GQA} & \multicolumn{4}{c}{\textbf{MoE Config}} \\
\cmidrule(r){1-5} \cmidrule(lr){6-6} \cmidrule(l){7-10}
\textbf{Layers} & \textbf{Hidden} & \textbf{FFN Dim} & \textbf{Vocab} & \textbf{Heads} & \textbf{Groups} & \textbf{MoE L} & \textbf{Exp FFN} & \textbf{Experts} & \textbf{Act Exp} \\
\midrule
4  & 256 & 896   & 200K & 2 & 2 & 2 & 384 & 256 & 8 \\
6  & 256 & 896   & 200K & 2 & 2 & 4 & 384 & 256 & 8 \\
8  & 384 & 1,344 & 200K & 3 & 3 & 6 & 576 & 256 & 8 \\
10 & 512 & 1,792 & 200K & 4 & 4 & 8 & 768 & 256 & 8 \\
\bottomrule
\end{tabular}
\end{adjustbox}
\end{subtable}

\vspace{1em}

\begin{subtable}{\textwidth}
\centering
\caption{Parameter counts and compute. MoE Act/Total = activated/total parameters per MoE layer. FLOPs/Token counts non-embedding forward-pass FLOPs.}
\label{tab:param_counts_fullmoe}
\begin{adjustbox}{max width=\textwidth}
\begin{tabular}{@{}ccccccc@{}}
\toprule
\multicolumn{4}{c}{\textbf{Layer Parameters}} & \multicolumn{2}{c}{\textbf{Total Parameters}} & \textbf{Compute} \\
\cmidrule(r){1-4} \cmidrule(lr){5-6} \cmidrule(l){7-7}
\textbf{Router} & \textbf{Dense Layer} & \textbf{MoE Act} & \textbf{MoE Total} & \textbf{Activated} & \textbf{Total} & \textbf{FLOPs/Token} \\
\midrule
65.5K & 0.95M & 2.69M & 75.8M & 110M & 256M  & 9.50e7 \\
65.5K & 0.95M & 2.69M & 75.8M & 115M & 408M  & 1.54e8 \\
98.3K & 2.14M & 6.00M & 171M  & 194M & 1.18B & 3.97e8 \\
131K  & 3.80M & 10.6M & 303M  & 298M & 2.64B & 8.15e8 \\
\bottomrule
\end{tabular}
\end{adjustbox}
\end{subtable}

\vspace{1em}

\begin{subtable}{\textwidth}
\centering
\caption{Training configuration. Token budgets and learning rates determined by scaling laws.}
\label{tab:training_config}
\begin{adjustbox}{max width=\textwidth}
\begin{tabular}{@{}ccccc@{}}
\toprule
\textbf{Tokens} & \textbf{Batch Size} & \textbf{Steps} & \textbf{LR} & \textbf{Total FLOPs} \\
\midrule
2.20B & 16 & 16,790 & 1.69e-2 & 2.09e17 \\
2.31B & 16 & 17,611 & 1.16e-2 & 3.54e17 \\
3.89B & 16 & 29,663 & 5.95e-3 & 1.54e18 \\
5.96B & 32 & 22,741 & 3.75e-3 & 4.86e18 \\
\bottomrule
\end{tabular}
\end{adjustbox}
\end{subtable}

\end{table*}

\subsection{Full MoE generalization results}
\label{app:fullmoe_results}

To verify that expert upcycling transfers beyond the interleaved architecture, we evaluate on a full MoE with 256 experts and top-$K{=}8$ ($\sim$3\% activation ratio), consistent with frontier MoE models~\citep{deepseekai2024deepseekv3, glm45, kimik2}. At the $\sim$1B total parameter scale with $\|g\|^2$ utility-based upcycling, the full MoE achieves strong gap closure across all model sizes tested (Table~\ref{tab:fullmoe_results}). Both interleaved and full MoE architectures show strong gap closure, confirming that expert upcycling is effective across MoE families and activation ratios.

\begin{table}[h]
\centering
\caption{Full MoE upcycling results (256$\to$512 experts, top-$K{=}8$, $\|g\|^2$ utility-based duplication) across model sizes from 154M to 1B total parameters. All values are validation loss ($\downarrow$).}
\label{tab:fullmoe_results}
\small
\setlength{\tabcolsep}{4pt}
\begin{tabular}{@{}cccccc@{}}
\toprule
\textbf{Active (M)} & \textbf{Total (M)} & \textbf{Fixed-256} & \textbf{Upcycled 512} & \textbf{Fixed-512} & \textbf{Eff.\ (\%)} \\
\midrule
7   & 154   & 3.564 & 3.519 & 3.516 & 93.8 \\
13  & 305   & 3.153 & 3.071 & 3.067 & 95.3 \\
40  & 1028  & 2.819 & 2.767 & 2.763 & 92.9 \\
\bottomrule
\end{tabular}
\end{table}

\clearpage
\section{Heuristic upcycling methods and results}
\label{app:heuristic_upcycling_methods}

This appendix provides the full description and experimental evaluation of heuristic expert and router upcycling methods referenced in \S~\ref{sec:results-utility}. We evaluated 10 expert-level and 10 router-level initialization heuristics designed to seed diversity among duplicates while retaining inherited capability. As reported in the main text, none of these heuristics meaningfully outperform simple copy-paste duplication.

\subsection{Summary of results}
\label{app:heuristic_results_summary}

Table~\ref{tab:heuristic_highlight} summarizes the best variant per method category on the 10-layer, 32$\to$64 expert interleaved MoE model. Across all expert and router heuristics, validation losses lie in a narrow band around the copy-paste baseline, with improvements of at most $\sim$$10^{-3}$. Several more aggressive methods (SVD mixing, orthogonalization) slightly degrade performance. These results indicate that \emph{maintaining a low initial loss} at the upcycling boundary is more important than introducing artificial diversity: perturbations can disrupt the pre-upcycling solution and force CPT to allocate capacity to recovery rather than specialization. In contrast, simple duplication provides a warm initialization, and loss-free load balancing ensures that all experts receive gradient signal, allowing training dynamics to drive expert differentiation naturally.

\begin{table}[t]
\centering
\caption{Heuristic upcycling: best variant per method category (10-layer, 32$\to$64 experts). No heuristic meaningfully outperforms copy-paste ($\leq 10^{-3}$ loss); several degrade performance. Full results in Appendix Tables~\ref{tab:expert_heuristics} and~\ref{tab:router_heuristics}.}
\label{tab:heuristic_highlight}
\small
\setlength{\tabcolsep}{4pt}
\begin{tabular}{@{}llc@{}}
\toprule
\textbf{Category} & \textbf{Best Variant} & \textbf{Val Loss} \\
\midrule
\multicolumn{3}{c}{\textit{Expert Initialization (baseline = 2.76858)}} \\
\midrule
Copy (baseline)    & --                      & 2.76858 \\
Scaled Copy        & $s{=}0.90$              & \textbf{2.76815} \\
Copy + Noise       & $\lambda{=}0.01$        & 2.76859 \\
Interpolate        & $\alpha{=}0.2$          & 2.76888 \\
Sparse Code Mix    & dict=1024               & 2.76938 \\
SVD Perturb        & values only             & 2.76938 \\
Drop Upcycle       & drop=0.3 (Xavier)       & 2.77043 \\
SVD Mix            & ratio=0.3               & 2.77472 \\
Orthogonal         & standard                & 2.77487 \\
\midrule
\multicolumn{3}{c}{\textit{Router Initialization (baseline = 2.76858)}} \\
\midrule
Copy (baseline)    & --                      & 2.76858 \\
Interpolate        & heavy                   & \textbf{2.76776} \\
SVD Perturb        & moderate                & 2.76816 \\
Copy + Noise       & very light              & 2.76819 \\
Bias Disc Dup      & all layers              & 2.76827 \\
Temp.\ Scaled      & sharp                   & 2.76847 \\
Adversarial        & light                   & 2.76848 \\
Bias Enc Dup       & all layers              & 2.77831 \\
\bottomrule
\end{tabular}
\end{table}

\subsection{Expert upcycling heuristics}
\label{app:expert_heuristics}

Table~\ref{tab:expert_method_desc} summarizes the 10 expert-level initialization heuristics evaluated. All methods upcycle from 32$\to$64 experts on the 10-layer interleaved MoE.

\begin{table}[h]
\centering
\caption{Expert upcycling heuristic descriptions and key hyperparameters.}
\label{tab:expert_method_desc}
\small
\begin{tabular}{@{}p{2.4cm}p{7.2cm}p{3.8cm}@{}}
\toprule
\textbf{Method} & \textbf{Description} & \textbf{Key Hyperparameters} \\
\midrule
Copy (baseline) & Duplicate each expert exactly, producing identical twins. & -- \\
Copy + Noise & Add calibrated Gaussian noise to duplicated expert weights. & $\lambda \in \{0.01, 0.02, 0.05\}$ \\
Drop Upcycle \citep{nakamura2025dropupcycling} & Re-initialize a fraction of weight-matrix columns while keeping the remainder from the original expert. & drop $\in \{0.3, 0.5, 0.7\}$; init: Xavier, Kaiming, Normal \\
Shuffle Columns & Randomly permute columns of weight matrices, preserving marginal statistics but changing connectivity. & -- \\
Interpolate & Create new experts by interpolating between adjacent experts: $w^{\text{new}}_{i} \leftarrow \alpha w_{i} + (1{-}\alpha) w_{i+1}$. & $\alpha \in \{0.2, 0.5, 0.7\}$ \\
Orthogonal & Gram--Schmidt orthogonalization to make duplicates orthogonal (in parameter space) to originals. & $\epsilon = 10^{-6}$ \\
Scaled Copy & Scale duplicated weights by a constant factor, changing magnitude while preserving direction. & $s \in \{0.90, 0.95, 1.05\}$ \\
SVD Perturb & Compute $W = USV^\top$ and perturb selected components (singular values, vectors, and/or drop small components). & $\sigma_v \!\in\! [0.05, 0.2]$, $\sigma_u \!\in\! [0.02, 0.1]$, drop $\!\in\! [0, 0.3]$ \\
SVD Mix & Compute SVDs for multiple experts and create hybrids by mixing singular vectors. & mix ratio $\in \{0.2, 0.3, 0.5, 0.7\}$ \\
Sparse Code Mix & Approximate $W \!\approx\! DC$ via ISTA-style sparse coding, then mix sparse codes between experts. & dict $\in \{256, 512, 1024\}$; sparsity $\in \{0.05, 0.2\}$ \\
\bottomrule
\end{tabular}
\end{table}

\subsection{Router upcycling heuristics}
\label{app:router_heuristics}

Table~\ref{tab:router_method_desc} summarizes the 10 router-level initialization heuristics evaluated on the same 10-layer model.

\begin{table}[h]
\centering
\caption{Router upcycling heuristic descriptions.}
\label{tab:router_method_desc}
\small
\begin{tabular}{@{}p{2.8cm}p{10.6cm}@{}}
\toprule
\textbf{Method} & \textbf{Description} \\
\midrule
Copy (baseline) & Duplicate router weights exactly. \\
Copy + Noise & Add Gaussian noise to duplicated router weights to seed routing diversity. \\
Interpolate & Interpolate router weights with neighbors: $r^{\text{dup}}_i \leftarrow \alpha r_i + (1{-}\alpha) r_{i+1}$. \\
Bias Only & Keep router weights identical but modify only biases to shift routing preferences. \\
Scaled Copy & Scale duplicate router weights to adjust routing sharpness/entropy. \\
Perturb New Only & Freeze original router weights and perturb only duplicates. \\
Orthogonal & Construct duplicate router weights orthogonal to originals. \\
Adversarial & Push duplicate router weights in the opposite direction of originals. \\
Temperature Scaled & Apply temperature scaling to duplicate router logits (pre-softmax) to control entropy. \\
SVD Perturb & SVD-based perturbations to router weights to preserve coarse routing structure while varying details. \\
\bottomrule
\end{tabular}
\end{table}

\subsection{Expert heuristic results}
\label{app:expert_heuristic_results}

Table~\ref{tab:expert_heuristics} reports validation loss for all expert upcycling heuristics on the 10-layer, 32-expert interleaved MoE model.

\begin{table*}[htbp]
\centering
\caption{Expert heuristic upcycling results (interleaved MoE, 10-layer, 32 experts). All methods upcycle from 32$\to$64 experts and train with identical CPT budget. Baseline copy-paste = 2.76858. Method descriptions and hyperparameters in Table~\ref{tab:expert_method_desc}.}
\label{tab:expert_heuristics}
\small
\begin{adjustbox}{max width=\textwidth}
\begin{tabular}{@{}llcc@{}}
\toprule
\textbf{Method} & \textbf{Qualifier / Variant} & \textbf{Key Param} & \textbf{Val Loss} \\
\midrule
\multicolumn{4}{c}{\textit{Copy-based}} \\
\midrule
Copy (baseline)       & --                    & --              & 2.76858 \\
Copy + Noise          & conservative          & $\lambda$=0.01  & 2.76859 \\
Copy + Noise          & moderate              & $\lambda$=0.02  & 2.76860 \\
Copy + Noise          & aggressive            & $\lambda$=0.05  & 2.76895 \\
Scaled Copy           & slight reduction      & $s$=0.95        & \textbf{2.76846} \\
Scaled Copy           & moderate reduction    & $s$=0.90        & \textbf{2.76815} \\
Scaled Copy           & slight amplification  & $s$=1.05        & 2.76896 \\
\midrule
\multicolumn{4}{c}{\textit{Drop upcycling}} \\
\midrule
Drop Upcycle          & conservative (Xavier)  & drop=0.3       & 2.77043 \\
Drop Upcycle          & moderate (Xavier)      & drop=0.5       & 2.77125 \\
Drop Upcycle          & aggressive (Xavier)    & drop=0.7       & 2.77304 \\
Drop Upcycle          & moderate (Kaiming)     & drop=0.5       & 2.77337 \\
Drop Upcycle          & moderate (Normal)      & drop=0.5       & 2.77203 \\
\midrule
\multicolumn{4}{c}{\textit{Interpolation}} \\
\midrule
Interpolate           & slight                & $\alpha$=0.2    & 2.76888 \\
Interpolate           & balanced              & $\alpha$=0.5    & 2.76953 \\
Interpolate           & heavy                 & $\alpha$=0.7    & 2.76966 \\
\midrule
\multicolumn{4}{c}{\textit{Orthogonal}} \\
\midrule
Orthogonal            & standard              & $\epsilon$=1e-6 & 2.77487 \\
\midrule
\multicolumn{4}{c}{\textit{SVD-based}} \\
\midrule
SVD Perturb           & conservative          & $\sigma_v$=0.05, $\sigma_u$=0.02 & 2.77476 \\
SVD Perturb           & moderate              & $\sigma_v$=0.1, $\sigma_u$=0.05  & 2.77483 \\
SVD Perturb           & aggressive            & $\sigma_v$=0.2, $\sigma_u$=0.1   & 2.77472 \\
SVD Perturb           & drop heavy            & drop=0.3        & 2.77441 \\
SVD Perturb           & values only           & $\sigma_v$=0.15 & 2.76938 \\
SVD Perturb           & vectors only          & $\sigma_u$=0.1  & 2.77468 \\
SVD Mix               & light                 & ratio=0.2       & 2.77483 \\
SVD Mix               & moderate              & ratio=0.3       & 2.77472 \\
SVD Mix               & heavy                 & ratio=0.5       & 2.77472 \\
SVD Mix               & aggressive            & ratio=0.7       & 2.77475 \\
\midrule
\multicolumn{4}{c}{\textit{Sparse Code Mix}} \\
\midrule
Sparse Code Mix       & small dict            & dict=256        & 2.77034 \\
Sparse Code Mix       & standard              & dict=512        & 2.76955 \\
Sparse Code Mix       & large dict            & dict=1024       & 2.76938 \\
Sparse Code Mix       & high sparsity         & sparsity=0.2    & 2.76981 \\
Sparse Code Mix       & low sparsity          & sparsity=0.05   & 2.76951 \\
Sparse Code Mix       & heavy mixing          & mix=0.5         & 2.76964 \\
Sparse Code Mix       & more iterations       & n\_iter=200     & 2.76945 \\
\bottomrule
\end{tabular}
\end{adjustbox}
\end{table*}

\subsection{Router heuristic results}
\label{app:router_heuristic_results}

Table~\ref{tab:router_heuristics} reports validation loss for all router upcycling heuristics on the same 10-layer model.

\begin{table*}[htbp]
\centering
\caption{Router heuristic upcycling results (interleaved MoE, 10-layer, 32 experts). All methods upcycle routers from 32$\to$64 expert slots with identical CPT budget. Baseline router copy = 2.76858. Methods are described in Appendix~\ref{app:heuristic_upcycling_methods}.}
\label{tab:router_heuristics}
\small
\begin{adjustbox}{max width=\textwidth}
\begin{tabular}{@{}llc@{}}
\toprule
\textbf{Method} & \textbf{Variant} & \textbf{Val Loss} \\
\midrule
\multicolumn{3}{c}{\textit{Copy-based}} \\
\midrule
Copy (baseline)        & --              & 2.76858 \\
Copy + Noise           & very light      & 2.76819 \\
Copy + Noise           & light           & 2.76844 \\
Copy + Noise           & moderate        & 2.76843 \\
Copy + Noise           & aggressive      & 2.76856 \\
\midrule
\multicolumn{3}{c}{\textit{Interpolation}} \\
\midrule
Interpolate            & light           & 2.76861 \\
Interpolate            & balanced        & 2.76877 \\
Interpolate            & heavy           & \textbf{2.76776} \\
\midrule
\multicolumn{3}{c}{\textit{Bias-based}} \\
\midrule
Bias Only              & all layers      & 2.76836 \\
Bias + Noise Only      & --              & 2.76845 \\
Bias Enc Dup           & --              & 2.77827 \\
Bias Enc Dup           & all layers      & 2.77831 \\
Bias Disc Dup          & --              & 2.76827 \\
Bias Disc Dup          & all layers      & 2.76827 \\
\midrule
\multicolumn{3}{c}{\textit{Scaling \& Temperature}} \\
\midrule
Scaled Copy            & very soft       & 2.76874 \\
Scaled Copy            & soft            & 2.76878 \\
Scaled Copy            & sharp           & 2.76860 \\
Temperature Scaled     & soft            & 2.76877 \\
Temperature Scaled     & very soft       & 2.76883 \\
Temperature Scaled     & sharp           & 2.76847 \\
\midrule
\multicolumn{3}{c}{\textit{Perturbation}} \\
\midrule
Perturb New Only       & light           & 2.76843 \\
Perturb New Only       & moderate        & 2.76860 \\
Perturb New Only       & aggressive      & 2.76858 \\
Orthogonal             & --              & 2.76867 \\
Adversarial            & light           & 2.76848 \\
Adversarial            & strong          & 2.76990 \\
\midrule
\multicolumn{3}{c}{\textit{SVD-based}} \\
\midrule
SVD Perturb            & conservative    & 2.76852 \\
SVD Perturb            & moderate        & \textbf{2.76816} \\
SVD Perturb            & aggressive      & 2.76826 \\
\bottomrule
\end{tabular}
\end{adjustbox}
\end{table*}

\section{Extended related work}
\label{app:extended_related_work}

This appendix provides detailed comparisons between expert upcycling and each line of related work, organized by category. For each cited paper, we explain the relationship to our contributions and highlight key differences.

\subsection{MoE foundations and scaling laws}

\paragraph{Sparsely-Gated MoE \citep{shazeer2017outrageously}.}
Introduced the sparsely-gated MoE layer with top-$K$ routing and load-balancing losses. Our work builds directly on this architecture: expert upcycling preserves the top-$K$ routing mechanism while expanding the expert pool, keeping per-token FLOPs constant.

\paragraph{GLaM \citep{du2022glam}.}
Demonstrated MoE scaling to 1.2T parameters with favorable quality-per-FLOP trade-offs. GLaM trains from scratch at the target scale; expert upcycling achieves similar capacity expansion goals but avoids the full from-scratch cost by growing an existing smaller MoE checkpoint.

\paragraph{Switch Transformers \citep{fedus2022switch} and GShard \citep{Lepikhin2020GShard}.}
Simplified MoE routing (top-1) and scaled expert parallelism across thousands of devices. These works focus on training infrastructure and routing simplification for from-scratch training. Our method is complementary: it provides an alternative path to large expert counts by growing mid-training rather than starting large.

\paragraph{Joint MoE Scaling Laws \citep{ludziejewski2025jointmoe}.}
Derived scaling laws jointly over active parameters, total parameters, and training tokens for MoE models, showing that MoE can be memory-efficient. These scaling laws directly motivate expert upcycling: they predict that increasing expert count at fixed active compute improves quality, and our method operationalizes this prediction without restarting training.

\paragraph{Fine-Grained MoE Scaling Laws \citep{krajewski2024finegrained}.}
Extended scaling analysis to fine-grained (many small) experts. Our experiments span activation ratios from 3\% to 50\%, covering both coarse and fine-grained regimes, and we show that upcycling efficiency is robust across this range.

\paragraph{Greater Leverage MoE Scaling Laws \citep{tian2025greaterleverage}.}
Identified sparsity as the most effective lever for improving MoE performance among MoE hyperparameters. This directly supports our approach: expert upcycling decreases the activation ratio (active-to-total parameter ratio) by adding experts while holding top-$K$ fixed.

\paragraph{Optimal Sparsity for MoE \citep{abnar2025optimalsparsity}.}
Derived compute-optimal sparsity schedules showing that the optimal number of experts depends on the compute budget. Our work complements this by providing a mechanism to \emph{change} the activation ratio mid-training as the compute budget evolves, rather than committing to a fixed activation ratio from the start.

\paragraph{Scaling Data-Constrained LMs \citep{muennighoff2023dataconstrained}.}
Showed that repeated data yields diminishing returns beyond $\sim$4 epochs, proposing scaling laws for data-constrained regimes. This is relevant to expert upcycling because our continued pre-training phase uses additional tokens; their findings inform how much additional data is needed for effective upcycling versus when returns diminish.

\subsection{Growing network size during training}

\paragraph{Net2Net \citep{Chen2016Net2Net}.}
Introduced function-preserving transforms (Net2WiderNet, Net2DeeperNet) for accelerating training via knowledge transfer. Expert upcycling is inspired by the function-preserving philosophy of Net2Net: our duplication + router expansion produces a warm initialization whose initial loss matches the parent model's loss (see \S~\ref{subsec:operator}). The key difference is that Net2Net grows dense width or depth, while we grow MoE expert count, a fundamentally different axis that exploits sparse activation.

\paragraph{Stacking Your Transformers \citep{du2024stackingtransformers}.}
Systematically studied model growth operators (depth stacking, width expansion) for efficient LLM pre-training, showing that stacking can save 50\%+ of training compute. Our work addresses a complementary growth axis: rather than stacking layers (depth), we duplicate experts (MoE width). The two approaches could potentially be composed for compound savings.

\paragraph{Stacking as Accelerated Gradient Descent \citep{agarwal2024stacking}.}
Provided optimization-theoretic justification for why layer stacking accelerates training, connecting it to accelerated gradient descent. Our regret-bound decomposition (Appendix~\ref{app:proof}) serves an analogous role for expert duplication, but the mechanisms differ: stacking exploits depth-wise structure, while our analysis addresses the capacity gap and initialization quality introduced by expert replication.

\paragraph{Deep Progressive Training \citep{bu2025dpt}.}
Analyzed progressive depth expansion through optimization theory and feature learning, establishing conditions for function-preserving growth to maintain convergence. Our work extends this progressive training philosophy to the MoE expert-count dimension, with an analogous theoretical framework analyzing warm initialization and symmetry breaking in the sparse routing setting.

\paragraph{RAPTR: Progressive Subnetwork Training \citep{panigrahi2024raptr}.}
Proposed training random subnetworks (depth-wise, width-wise) and progressively increasing subnetwork size, showing this generalizes and fixes issues with layer dropping. RAPTR operates within a fixed architecture by training subsets; expert upcycling instead \emph{expands} the architecture by adding new experts. The approaches are complementary: RAPTR could be applied during the continued pre-training phase after upcycling.

\paragraph{Multi-linear Operators for Model Reuse \citep{pan2023multilinear}.}
Proposed correlating each weight of a target model to all weights of a pretrained model via multi-linear operators, capturing inter- and intra-weight interactions. This addresses dense model growth with full weight correlation. Expert upcycling takes a simpler approach (exact duplication) that achieves warm initialization by construction, avoiding the computational overhead of learning cross-weight mappings, and operates in the MoE setting where sparse routing provides a natural growth axis.

\subsection{Upcycling from dense checkpoints}

\paragraph{Sparse Upcycling \citep{komatsuzaki2022sparse}.}
The foundational work on converting dense checkpoints to MoE by replicating the FFN into multiple experts. The critical distinction from our work: Sparse Upcycling performs a \emph{dense $\rightarrow$ MoE} transition, while expert upcycling performs \emph{MoE $\rightarrow$ larger MoE} expansion. This difference has significant implications: (i) our starting point already has trained routing and specialized experts, (ii) we preserve the existing sparse computation pattern, and (iii) our method enables iterative capacity expansion without architectural regime changes.

\paragraph{Drop-Upcycling \citep{nakamura2025dropupcycling}.}
Improves upon Sparse Upcycling by partially re-initializing expert weights during the dense-to-MoE conversion, promoting diversity and maintaining a learning curve slope similar to from-scratch MoE training. Our experiments with heuristic initialization methods (Appendix~\ref{app:heuristic_upcycling_methods}) include analogous drop-based strategies along with nine other approaches (noise injection, SVD perturbation, orthogonalization, interpolation, sparse code mixing, and others); we find that all heuristic methods yield negligible gains ($<$1e-3 validation loss) over simple copy-paste for MoE-to-MoE upcycling. This suggests that in the already-sparse setting, the trained router provides sufficiently strong symmetry-breaking signals during CPT, making elaborate initialization diversity unnecessary.

\paragraph{Scaling Laws for Upcycling MoE \citep{liew2025scalinglaws}.}
Derived scaling laws specifically for the dense-to-MoE upcycling transition, revealing a critical interaction term: upcycling efficiency \emph{decreases} with longer dense pre-training because the sunk dense tokens slow subsequent MoE training progress. Our MoE-to-MoE setting exhibits the opposite behavior: upcycling efficiency \emph{increases} with pre-training duration (Table~\ref{tab:budget_highlight}a, transition-timing sweep), because the base MoE already possesses sparse routing structure and specialized experts that transfer directly to the expanded architecture. This qualitative reversal underscores that dense-to-MoE and MoE-to-MoE upcycling are governed by fundamentally different dynamics.

\paragraph{BAM! Parameter Upcycling \citep{zhang2024bam}.}
Explored simple and efficient parameter upcycling strategies for creating MoE from dense models, finding that straightforward approaches can be surprisingly effective. Our findings align: simple expert duplication (COPY) is competitive with more elaborate heuristics for MoE-to-MoE growth. However, BAM focuses on the dense-to-MoE transition while we address the already-sparse setting.

\paragraph{Upcycling LLMs into MoE \citep{he2024upcyclingllm}.}
Studied upcycling at scale (NVIDIA), examining router design choices and expert granularity for converting dense LLMs to MoE. Their work provides practical recipes for the dense-to-MoE transition at production scale. Our contribution is orthogonal: we address the next step, growing an already-sparse model to have more experts.

\paragraph{DeRS: Efficient Upcycled MoE \citep{huang2025ders}.}
Proposed decomposing upcycled MoE experts into shared base weights and lightweight delta weights for parameter efficiency, applicable to both training and compression. DeRS addresses parameter redundancy \emph{within} upcycled experts; our work addresses how to \emph{create} more experts from existing ones. The two approaches could be combined: one could apply DeRS compression after expert upcycling to reduce the parameter overhead of the expanded model.

\paragraph{Nexus: Adaptive Upcycling \citep{gritsch2025nexus}.}
Introduced an adaptive router with domain embeddings that can incrementally integrate new experts trained on new domains. Nexus and expert upcycling share the goal of expanding MoE capacity post-training, but differ in mechanism: Nexus adds independently trained domain-specific experts with a specialized router, while we duplicate existing experts and rely on continued pre-training for specialization. Our approach does not require domain-specific expert training and works with standard MoE routers.

\paragraph{LayerMoE \citep{zhang2025layermoe}.}
Expands multilingual MoE capacity by adding new experts to an existing MoE backbone, with a layer-wise allocation algorithm that decides how many new experts each layer needs based on language-specific representation characteristics. Tokens from old languages are routed to the original experts while new-language tokens use the added experts, avoiding catastrophic forgetting. LayerMoE adds new, independently-trained experts targeted at new signal (new languages); expert upcycling instead duplicates existing experts and relies on continued pre-training for specialization on the same data distribution.

\paragraph{Branch-Train-Stitch \citep{zhang2025bts}.}
Combines independently trained, domain-specialized LLM experts into a generalist via lightweight ``stitch layers'' inserted between frozen experts and a seed LLM. New experts can be added with minimal retraining of the stitch layers, supporting flexible capacity expansion. BTS adds new, externally-trained experts that bring new domain signal; expert upcycling reuses existing experts in place and does not require parallel domain-specific training runs.

\paragraph{Branch-Train-MiX \citep{sukhbaatar2024btx}.}
Trains domain-specialized expert LLMs in embarrassingly parallel fashion, then merges them into an MoE with learned routing. BTX creates expert diversity through independent training on different data domains; expert upcycling creates diversity through duplication followed by symmetry breaking during joint continued pre-training. BTX requires parallel training infrastructure for each expert branch, while our method operates on a single training run.

\paragraph{Symphony-MoE \citep{wang2026symphonymoe}.}
Constructs an MoE from multiple disparate pretrained models (e.g., Qwen2 + Qwen2.5-Coder) via functional alignment using neuron permutation and SLERP-based parameter merging. Symphony-MoE maximizes initial expert diversity by leveraging independently trained models, but requires access to multiple compatible architectures and a training-free harmonization stage to resolve parameter space misalignment. Expert upcycling requires only a single MoE checkpoint and achieves diversity through continued training rather than multi-source assembly, making it applicable in settings where diverse pretrained models are unavailable.

\paragraph{Reuse, Don't Retrain \citep{parmar2024reuse}.}
Provided practical guidelines for continued pre-training of language models, covering data distribution design and learning rate schedules. These recipes are complementary to expert upcycling: they inform how to design the continued pre-training phase after expert expansion. Our work focuses on the architectural growth step itself rather than the training recipe.

\subsection{Expert specialization, diversity, and routing}

\paragraph{Representation Collapse in MoE \citep{chi2022representationcollapse}.}
Identified and addressed the representation collapse problem where MoE routing encourages token clustering around expert centroids, proposing hyperspherical routing as a solution. This is relevant to expert upcycling because duplication initially creates identical expert representations; our theoretical analysis of symmetry breaking explains how continued training escapes this degenerate state.

\paragraph{Grove MoE \citep{wu2025grovemoe}.}
Introduced a new MoE architecture with heterogeneous expert sizes (adjugate experts) and a dynamic per-token activation mechanism; the paper's primary contribution is architectural. Their 33B-parameter instantiation (GroveMoE-Base) is produced by applying an upcycling strategy to the Qwen3-30B-A3B-Base MoE checkpoint during mid-training, but upcycling is a training strategy in their setup rather than the main contribution. Grove MoE therefore operates on a different axis than expert upcycling: it modifies expert sizes and activation dynamics, while expert upcycling modifies expert count with homogeneous experts and fixed top-$K$ routing. The two are orthogonal and could be combined.

\paragraph{MoE Design Choices \citep{fan2024empiricalmoe}.}
Systematically ablated MoE design choices, finding that routing granularity (token-level vs.\ sequence-level) has the largest impact on performance, while expert collapse does not necessarily hurt validation perplexity. Their finding that top-2 token-level routing is the only configuration surpassing the dense baseline with equivalent total parameters motivates our choice of top-$K$ values across experiments (top-$K{=}2$ for the interleaved MoE main result and top-$K{=}8$ for the full MoE result, matching frontier configurations).

\subsection{MoE deployment and saliency metrics}

\paragraph{Expert Pruning and Skipping \citep{lu2024experts}.}
Developed methods for reducing MoE inference cost by pruning or dynamically skipping experts. Expert upcycling and expert pruning are natural duals: upcycling adds capacity for quality improvement, pruning removes capacity for efficiency. Our utility-based expert selection identifies the most important experts to duplicate, concentrating added capacity where the loss is most sensitive.

\paragraph{Saliency metrics \citep{han2015learning, lecun1989optimal, hassibi1992second, molchanov2017taylor, soen2024tradeoffs, li2025fishersforfree}.}
The pruning literature provides a rich toolkit of saliency metrics. We repurpose these tools for capacity expansion: gradient-based importance scores identify which experts contribute most to the loss landscape, and we preferentially duplicate these high-utility experts.

\subsection{Conditional compute and dynamic routing}

\paragraph{Mixture-of-Depths \citep{raposo2024mixtureofdepths}.}
Mixture-of-Depths (MoD) routes tokens to skip transformer layers entirely, varying the compute depth per token rather than routing to different expert FFNs. Like MoE, MoD decouples total model capacity from per-token compute, but through a depth-routing mechanism rather than expert selection. Expert upcycling is orthogonal: we expand the number of experts in an MoE model, not the depth of computation per token. The two approaches could in principle be combined.

\subsection{Continual and lifelong learning}

\paragraph{Continual pre-training and plasticity.}
Continual learning studies how models can acquire new knowledge without forgetting previously learned representations, with plasticity (the ability to continue learning) as a central concern~\citep{kirkpatrick2017ewc, gupta2023continual}.
Expert upcycling involves a form of continued pre-training on new data after an architectural change, which raises related questions: do duplicated experts retain the source model's representations while also specializing on new data?
Our experiments show that the upcycled model maintains most of the from-scratch baseline's capability across 11 downstream benchmarks (within 0.3 points on average at 100\% CPT), suggesting that catastrophic forgetting is not a significant concern in our setting.
We attribute this to the warm initialization: duplicated experts start from trained weights, so the model does not need to relearn existing knowledge from scratch.

\end{document}